\documentclass[10pt]{article} 

\usepackage[accepted]{tmlr}

\usepackage{amsmath,amsfonts,bm}

\def\eqref#1{equation~\ref{#1}}

\def\1{\bm{1}}

\DeclareMathAlphabet{\mathsfit}{\encodingdefault}{\sfdefault}{m}{sl}
\SetMathAlphabet{\mathsfit}{bold}{\encodingdefault}{\sfdefault}{bx}{n}

\newcommand{\R}{\mathbb{R}}

\DeclareMathOperator*{\argmin}{arg\,min}

\usepackage[utf8]{inputenc} 
\usepackage[T1]{fontenc}    
\usepackage[hidelinks]{hyperref}       
\usepackage{url}            
\usepackage{booktabs}       
\usepackage{amsfonts}       
\usepackage{nicefrac}       
\usepackage{microtype}      
\usepackage{xcolor}         
\usepackage{multirow,multicol}
\usepackage{amsmath,amsfonts,amsthm}
\usepackage{enumerate}
\usepackage{paralist}
\usepackage{enumitem}
\usepackage{adjustbox}

\usepackage{appendix}
\usepackage{cleveref}
\usepackage{thmtools}
\usepackage{mathtools}
\usepackage{algorithm}
\usepackage{algpseudocode}
\usepackage{dsfont}
\usepackage{tikz}
\usepackage{pgfplots}
\usepackage{pgfplotstable}
\usepackage{graphicx}
\pgfplotsset{compat=1.3} 
\usepackage{caption}
\usepackage{subcaption}
\usepackage{bbm}

\usepackage[normalem]{ulem}

\renewcommand{\det}{\mathit{det}}

\newcommand{\CF}{\mathit{CF}}

\renewcommand{\cal}[1]{\mathcal{#1}}

\usepackage{xspace}
\usepackage{adjustbox}
\usepackage{makecell}

\newcommand{\Step}{\textrm{StEP}\xspace}

\newcommand{\DICE}{\textrm{DiCE}\xspace}
\newcommand{\FACE}{\textrm{FACE}\xspace}
\newcommand{\CCHVAE}{\textrm{C-CHVAE}\xspace}
\newcommand{\euler}{\mathrm{e}}
\newcommand{\Adult}{\texttt{UCI Adult}\xspace}

\newcommand*\rot{\rotatebox[origin=c]{70}}

\renewcommand{\cal}[1]{\mathcal{#1}}
\newcommand{\poi}{\vec x}

\newtheorem{theorem}{Theorem}[section]
\newtheorem{lemma}[theorem]{Lemma}

\theoremstyle{definition}

\theoremstyle{definition}

\newtheorem{example}[theorem]{Example}

\usepackage[symbol]{footmisc}

\title{Simple Steps to Success: A Method for Step-Based \\
Counterfactual Explanations}

\author{\name Jenny Hamer\thanks{equal contribution} \email hamer@google.com \\
      \addr Google DeepMind, New York
      \AND
      \name Nicholas Perello\footnotemark[1] \email nperello@umass.edu \\
      \addr University of Massachusetts, Amherst
      \AND
      \name Jason Valladares\thanks{work completed during Master's at the University of Massachusetts, Amherst} \email jason.valladares3@gmail.com \\
      \addr Google 
      \AND
      Vignesh Viswanathan\footnotemark[1] \email vviswanathan@umass.edu\\
      \addr University of Massachusetts, Amherst
      \AND
      Yair Zick \email yzick@umass.edu\\
      \addr University of Massachusetts, Amherst
      }

\def\month{10}  
\def\year{2024} 
\def\openreview{\url{https://openreview.net/forum?id=R6ey5DKaoX}} 

\begin{document}

\maketitle

\begin{abstract}
Algorithmic recourse is a process that leverages counterfactual explanations, going beyond understanding why a system produced a given classification, to providing a user with actions they can take to change their predicted outcome.
Existing approaches to compute such interventions---known as {\em recourse}---identify a set of points that satisfy some desiderata---e.g. an intervention in the underlying causal graph, minimizing a cost function, etc. 
Satisfying these criteria, however, requires extensive knowledge of the underlying model structure, an often unrealistic amount of information in several domains. 
We propose a data-driven and model-agnostic framework to compute counterfactual explanations. 
We introduce \Step, a computationally efficient method that offers \emph{incremental steps} along the data manifold that directs users towards their desired outcome. 
We show that \Step uniquely satisfies a desirable set of axioms. Furthermore, via a thorough empirical and theoretical investigation, we show that \Step offers provable robustness and privacy guarantees while outperforming popular methods along important metrics. 
\end{abstract}

\section{Introduction}\label{sec:intro}
An automatic decision maker produces a negative prediction for some user---e.g. denies their grad school application, offers them bad loan terms or an overly strict criminal sentence; what can the user do to change this outcome? 
Counterfactual explanations \citep{wachter2017counterfactual} recommend actions that change algorithmic predictions on a given point. 
This is usually modeled as a constrained optimization problem that outputs specific points which satisfy certain desirable properties: actionability, validity, data manifold closeness and causality to name a few \citep{verma2020counterfactual}. 
Achieving these desiderata often requires both significant compute power and user/model information. 
We propose a \textbf{lightweight algorithm for producing counterfactual explanations}; rather than searching for good interventions for users, we search for good \emph{directions} that users can take. 
We use these directions to create an iterative recourse mechanism where stakeholders can repeatedly request new directions after carrying out the recommended changes. 
We show that by carefully choosing these directions, we satisfy several desirable properties of algorithmic recourse at a significantly lower computational cost.

\subsection{Our Contributions}
We propose a recourse algorithm called {\em Stepwise Explainable Paths} (\Step). 
Our key theoretical insight is that \Step is the \emph{only method} for generating recourse directions (counterfactual explanations) which satisfies a set of natural properties. 
\Step directions are model-agnostic and easy to compute, requiring only the training dataset and the output of the model of interest on points in the training dataset. 
That is, our method does not require prior knowledge of the underlying model architecture and instead takes a strongly data-dependent approach. 
In addition to introducing \Step, a novel step-based recourse method, we present its provable quality (\Cref{subsec:step}), diversity (\Cref{subsec:step-clustering}), and privacy (\Cref{sec:step-privacy}) guarantees. We also provide an extensive experimental evaluation of \Step, including a holistic cross-comparison with three popular recourse methods (\DICE \citep{Mothilal2020Dice}, \FACE \citep{poyiadzi2020face}, and \CCHVAE \citep{pawelczyk2020cchvae}) on three widely-used financial datasets---Credit Card Default \citep{yeh2009creditcarddefault}, Give Me Some Credit \citep{credit2011givemesomecredit} and UCI Adult \citep{Kohavi1996Adult} datasets. We also investigate \Step's robustness to noise (\Cref{subsec:robustness}).

\subsection{Related Work}\label{sec:related}
Counterfactual explanations are a fundamental concept in the model explanation literature \citep{verma2020counterfactual}. 
First proposed by \cite{wachter2017counterfactual}, they are founded in legal interpretations of explainability \citep{wachter2018counterfactual} and are distinct from \emph{feature highlighting} methods \citep{barocas2020counterfactuals}, such as Shapley value-based methods \citep{Datta2016,lundberg2017unexpected}, Local Interpretable Model-Agnostic Explanations \citep{Ribeiro2016should} and saliency maps \citep{simonyan2013deep}. 
Recent policy efforts describe algorithmic recourse as a means for fostering trust in AI systems \citep{wachter2018counterfactual,ntia2023policy,biden2023execorderAI}.
Whereas feature highlighting methods indicate important features (or feature interactions \citep{patel2021high}), counterfactual explanations identify \emph{changes} to features that are likely to change the outcome\citep{karimi2020survey,verma2020counterfactual}. 
These changes are usually solutions to an optimization problem, ensuring that the explanations are valid, actionable, sparse and diverse \citep{Mothilal2020Dice, Karimi2020Mace,Ustun2019Recourse}. 
These solutions are computed using integer linear programs \citep{Ustun2019Recourse, Kentaro2020Dace}, SAT solvers \citep{Karimi2020Mace}, or gradient descent \citep{Mothilal2020Dice, wachter2017counterfactual}. 
Other works provide a sequence of steps from the point of interest to a point with the desirable outcome along the data manifold \citep{poyiadzi2020face}. 
While these prior methods offer reasonable approaches to recourse generation, none of them axiomatically characterize their approach.

Another approach in the literature is to solve the causal problem of finding the best intervention \citep{Karimi2021Recourse}. 
\citet{Karimi2021Recourse} argue that any recourse recommendation must be consistent with the underlying causal relations between variables. 
However, this requires complete (or, as in \citet{Karimi2020Imperfect}, imperfect) knowledge of the underlying causal model, which is often practically or computationally infeasible. 

We address this issue by providing users with promising directional actions, allowing them more flexibility and agency in enacting recourse recommendations.
Our axiomatic characterization is similar to that of Monotone Influence Measures (MIM) \citep{sliwinski2019mim}. 
However, \citeauthor{sliwinski2019mim}'s approach is not iterative in nature, nor is their recommended direction guaranteed to change the prediction outcome.

\section{Preliminaries}\label{sec:prelims}
We introduce some general notation and definitions used throughout this work. We denote by $\poi \in \R^m$ a {\em point of interest} (PoI) corresponding to an individual or their current state, and define a dataset $\cal X = \{\vec x^1, \vec x^2, \dots, \vec x^n\} \subseteq \R^m$ with $m$ features and $n$ datapoints.
We use $x_i$ to denote the $i$-th index of the vector $\vec x$. 

We focus on binary classification using a {\em model of interest} $f: \R^m \mapsto \pm 1$ trained on the dataset $\cal X$; our goal is to produce a counterfactual explanation for $\poi \notin \cal X$ where $f(\poi) = -1$.
The {\em counterfactual explanation} for a point of interest $\poi$ is a direction (or a set of directions) $\vec d\in \R^m$  that moves the point $\poi$ towards the positive class, i.e. $f(\poi+ c \cdot \vec d) = 1$ for some positive value $c>0$. We go beyond simply providing an explanation for the the PoI's outcome by framing our output as recommended {\em recourse} which can be actioned by the user.
Upon applying recourse to a PoI $\poi$, we refer to $\poi + c \cdot \vec d$ as a {\em counterfactual} (CF) of $\poi$ denoted by $\poi_{\CF}$.

Offering a single explanation may not be helpful in real-world settings; users may be unable to make a single dramatic change to their features or may not exactly follow the suggested recourse. 
We allow the stakeholder to repeatedly request new recourse directions as they change their values until they are positively classified.
We call such an approach {\em direction-based recourse}.
To illustrate this approach, consider the following example of algorithmic loan approval.

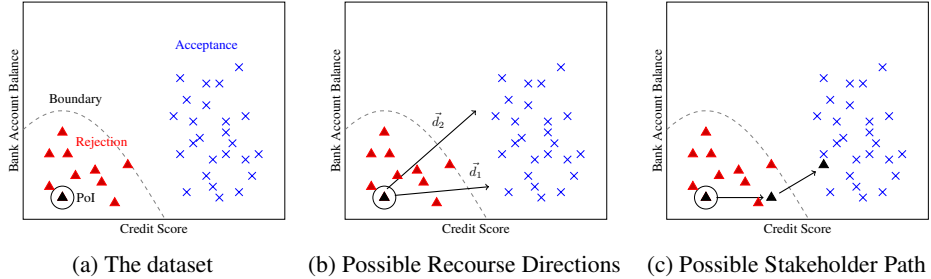
\begin{figure*}
    \centering
    \begin{subfigure}[t]{0.3\textwidth}
         \centering
         \resizebox{0.9\textwidth}{!}{
         \begin{tikzpicture}
        \begin{axis}[
        xmin = 0, 
        xmax = 10, 
        ymin = 0, 
        ymax = 10, 
        ticks = none, 
        xlabel={Credit Score}, 
        ylabel={Bank Account Balance}, 
        ]

        \addplot+[gray, no marks, dashed, domain = 0:5.4] {5*cos(deg((x-1.5)/2.5))};
        \node at (20,55) {Boundary};

        \addplot+[only marks, red, mark=triangle*, mark size = 4pt]
        coordinates {(1,1.5)(1.5, 1)(2,2)(2.75, 2.25)(3, 1.7)(1.7, 3)(1.5, 4)(3.5, 0.75)(4, 2.5)(1,3)};
        \node at (30,35) {\textcolor{red}{Rejection}};

        \addplot+[only marks, black, mark = triangle*, mark size = 4pt, mark options = {fill=black}]
        coordinates {(1.5, 1)};


        \addplot+[only marks, blue, mark=x, mark size = 4pt]
        coordinates {(6, 3)(7, 2)(7.25, 1.5)(6.5, 3.5)(6.25, 1.25)(7.5, 1)(7.75,3)(7.25, 2.75)(8.25,1.25)(8.5,2.75)(8.75, 2)(9, 3)(8, 3.5)(7.75, 4)
        (5.75,4.75)(6.75, 3.75)(6.5,4.5)(7, 5.25)(7.75, 4.5)(6.25, 5.5)(7, 6.25)(6, 6.5)(8.5, 4.75)(7.5,6.25)(8.25, 7)
        };
        \node at (70,80) {\textcolor{blue}{Acceptance}};
       \draw (15,10) circle (0.3cm);
       \node at (24,10) {PoI};
        \end{axis}

    \end{tikzpicture}}
         \caption{The dataset}
         \label{fig:loan-example-dataset}
     \end{subfigure}
     \begin{subfigure}[t]{0.3\textwidth}
         \centering
         \resizebox{0.9\textwidth}{!}{
         \begin{tikzpicture}
        \begin{axis}[
        xmin = 0, 
        xmax = 10, 
        ymin = 0, 
        ymax = 10, 
        ticks = none, 
        xlabel={Credit Score}, 
        ylabel={Bank Account Balance}, 
        legend pos=north east, 
        ]

        \addplot+[gray, no marks, dashed, domain = 0:5.4] {5*cos(deg((x-1.5)/2.5))};

        \addplot+[only marks, red, mark=triangle*, mark size = 4pt]
        coordinates {(1,1.5)(1.5, 1)(2,2)(2.75, 2.25)(3, 1.7)(1.7, 3)(1.5, 4)(3.5, 0.75)(4, 2.5)(1,3)};

        \addplot+[only marks, black, mark = triangle*, mark size = 4pt, mark options = {fill=black}]
        coordinates {(1.5, 1)};


        \addplot+[only marks, blue, mark=x, mark size = 4pt]
        coordinates {(6, 3)(7, 2)(7.25, 1.5)(6.5, 3.5)(6.25, 1.25)(7.5, 1)(7.75,3)(7.25, 2.75)(8.25,1.25)(8.5,2.75)(8.75, 2)(9, 3)(8, 3.5)(7.75, 4)
        (5.75,4.75)(6.75, 3.75)(6.5,4.5)(7, 5.25)(7.75, 4.5)(6.25, 5.5)(7, 6.25)(6, 6.5)(8.5, 4.75)(7.5,6.25)(8.25, 7)
        };
        \draw[thick, ->] (19, 11) -- (55, 15);
        \node[anchor=south] at (50, 17){%
        $\vec d_1$
      };
      \draw[thick, ->] (16, 14) -- (50, 50);
        \node[anchor=south east] at (40, 40){%
        $\vec d_2$
      };
       \draw (15,10) circle (0.3cm);
        \end{axis}

    \end{tikzpicture}}
         \caption{Possible Recourse Directions}
         \label{fig:loan-example-directions}
     \end{subfigure}
     \begin{subfigure}[t]{0.3\textwidth}
         \centering
         \resizebox{0.9\textwidth}{!}{
         \begin{tikzpicture}
        \begin{axis}[
        xmin = 0, 
        xmax = 10, 
        ymin = 0, 
        ymax = 10, 
        ticks = none, 
        xlabel={Credit Score}, 
        ylabel={Bank Account Balance}, 
        legend pos=north east, 
        ]

        \addplot+[gray, no marks, dashed, domain = 0:5.4] {5*cos(deg((x-1.5)/2.5))};

        \addplot+[only marks, red, mark=triangle*, mark size = 4pt]
        coordinates {(1,1.5)(1.5, 1)(2,2)(2.75, 2.25)(3, 1.7)(1.7, 3)(1.5, 4)(3.5, 0.75)(4, 2.5)(1,3)};


        \addplot+[only marks, black, mark = triangle*, mark size = 4pt, mark options = {fill=black}]
        coordinates {(1.5, 1)(4, 1)(6, 2.5)};


        \addplot+[only marks, blue, mark=x, mark size = 4pt]
        coordinates {(6, 3)(7, 2)(7.25, 1.5)(6.5, 3.5)(6.25, 1.25)(7.5, 1)(7.75,3)(7.25, 2.75)(8.25,1.25)(8.5,2.75)(8.75, 2)(9, 3)(8, 3.5)(7.75, 4)
        (5.75,4.75)(6.75, 3.75)(6.5,4.5)(7, 5.25)(7.75, 4.5)(6.25, 5.5)(7, 6.25)(6, 6.5)(8.5, 4.75)(7.5,6.25)(8.25, 7)
        };
        \draw[thick, ->] (19, 10) -- (37, 10);
        \draw[thick, ->] (43, 12) -- (57, 22);

       \draw (15,10) circle (0.3cm);
        \end{axis}

    \end{tikzpicture}}
         \caption{Possible Stakeholder Path}
         \label{fig:loan-example-path}
     \end{subfigure}
     
    \caption{(Left) The training dataset that the loan approval algorithm uses, (Middle) Plotting possible recourse directions for Example \ref{ex:loan-recourse} and (Right) plotting a stakeholder path toward loan acceptance.}
    \label{fig:loan-example}
\end{figure*}

\begin{example}\label{ex:loan-recourse}
Consider a loan applicant whose application is rejected by an algorithm utilizing two factors: bank account balance and credit score. 
The loan applicant, i.e., the PoI, the algorithm, and the training dataset are described in \Cref{fig:loan-example-dataset}.
Several directions can change the PoI's value from $\triangle$ (red points with label $-1$) to $\times$ (blue points with label $+1$). 
For example, increasing the credit score while leaving the bank account balance unchanged eventually changes the applicant's label; increasing both bank balance and credit score (\Cref{fig:loan-example-directions}) results in a positive outcome. 
Suppose that we provide the applicant with both alternatives; the applicant makes some modifications and reapplies for a loan. 
If the application is rejected, we provide them with new directions, based on their current state, and repeat until the application is accepted. 
This results in a recourse \emph{path}, rather than a single direction (\Cref{fig:loan-example-path}).
\end{example}

We wish to ensure that our directions are robust; thus, even if the stakeholder somewhat deviates from our suggested path, they are still likely to secure a positive outcome. 
This also offloads some of the computational costs onto the stakeholder, which guarantees certain recourse properties that would otherwise require a significant amount of information and computational cost. 
Consider Example \ref{ex:loan-recourse}: after providing two different directions, the change that the applicant makes to their datapoint will likely be a minimum cost change that approximately follows one of the directions we propose in  Figure \ref{fig:loan-example-directions}. 
Even if the resulting change does not facilitate a change in outcome, or is incorrectly executed, we can still offer additional directions until a desirable outcome is obtained. 

\subsection{Data-driven Recourse}\label{subsec:model-driven-recourse}
Whether explanations should be computed based on the underlying model or the observed data is a highly debated topic \citep{chen2020truetomodel,barocas2020counterfactuals,janzing2020causality}.
While data-driven recourse methods exist \citep{poyiadzi2020face}, most work solely with the model of interest $f$ \citep{Mothilal2020Dice, wachter2017counterfactual, Karimi2020Mace}.

Recent work shows that explanations can perform poorly when they are inconsistent with the data manifold---i.e. the underlying distribution from which the data is drawn \citep{Frye2021DataManifold, Aas2021Shapley} and are vulnerable to manipulation \citep{slack2020fooling, slack2021counterfactual}. 
Ignoring the data manifold when computing counterfactuals can be even more pernicious, resulting in recourse recommendations that may not improve the outcome in any way, but simply move it outside the data manifold. 
To see why this is the case, consider a simplified, two dimensional instance as shown in Example \ref{ex:loan-recourse} but with a different point of interest (given in \Cref{fig:loan-example-2-dataset}). 
The shortest distance perturbation which crosses the decision boundary corresponds to the direction $\vec d_1$ (as given in \Cref{fig:loan-example-2-directions}). 
This direction suggests that the stakeholder decrease their credit score and marginally increase their bank account balance, which actively moves the stakeholder away from the positively classified points. 
An explanation algorithm that recognizes the data manifold will offer a direction similar to $\vec d_2$, arguably making the stakeholder a better candidate for loan approval.  
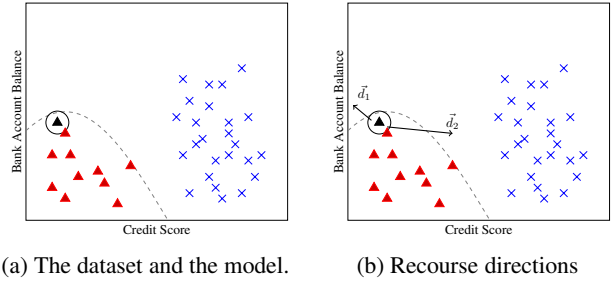
\begin{figure*}
    \centering
    \begin{subfigure}[b]{0.3\textwidth}
         \centering
         \resizebox{0.9\textwidth}{!}{
         \begin{tikzpicture}
        \begin{axis}[
        xmin = 0, 
        xmax = 10, 
        ymin = 0, 
        ymax = 10, 
        ticks = none, 
        xlabel={Credit Score}, 
        ylabel={Bank Account Balance}, 
        legend pos=north east, 
        ]

        \addplot+[gray, no marks, dashed, domain = 0:5.4] {5*cos(deg((x-1.5)/2.5))};

        \addplot+[only marks, red, mark=triangle*, mark size = 4pt]
        coordinates {(1,1.5)(1.5, 1)(2,2)(2.75, 2.25)(3, 1.7)(1.7, 3)(1.5, 4)(3.5, 0.75)(4, 2.5)(1,3)};


        \addplot+[only marks, black, mark = triangle*, mark size = 4pt, mark options = {fill=black}]
        coordinates {(1.2, 4.5)};


        \addplot+[only marks, blue, mark=x, mark size = 4pt]
        coordinates {(6, 3)(7, 2)(7.25, 1.5)(6.5, 3.5)(6.25, 1.25)(7.5, 1)(7.75,3)(7.25, 2.75)(8.25,1.25)(8.5,2.75)(8.75, 2)(9, 3)(8, 3.5)(7.75, 4)
        (5.75,4.75)(6.75, 3.75)(6.5,4.5)(7, 5.25)(7.75, 4.5)(6.25, 5.5)(7, 6.25)(6, 6.5)(8.5, 4.75)(7.5,6.25)(8.25, 7)
        };
       \draw (12,45) circle (0.3cm);
        \end{axis}

    \end{tikzpicture}}
         \caption{The dataset and the model.}
         \label{fig:loan-example-2-dataset}
     \end{subfigure}
     \begin{subfigure}[b]{0.3\textwidth}
         \centering
         \resizebox{0.9\textwidth}{!}{
         \begin{tikzpicture}
        \begin{axis}[
        xmin = 0, 
        xmax = 10, 
        ymin = 0, 
        ymax = 10, 
        ticks = none, 
        xlabel={Credit Score}, 
        ylabel={Bank Account Balance}, 
        legend pos=north east, 
        ]

        \addplot+[gray, no marks, dashed, domain = 0:5.4] {5*cos(deg((x-1.5)/2.5))};

        \addplot+[only marks, red, mark=triangle*, mark size = 4pt]
        coordinates {(1,1.5)(1.5, 1)(2,2)(2.75, 2.25)(3, 1.7)(1.7, 3)(1.5, 4)(3.5, 0.75)(4, 2.5)(1,3)};


        \addplot+[only marks, black, mark = triangle*, mark size = 4pt, mark options = {fill=black}]
        coordinates {(1.2, 4.5)};


        \addplot+[only marks, blue, mark=x, mark size = 4pt]
        coordinates {(6, 3)(7, 2)(7.25, 1.5)(6.5, 3.5)(6.25, 1.25)(7.5, 1)(7.75,3)(7.25, 2.75)(8.25,1.25)(8.5,2.75)(8.75, 2)(9, 3)(8, 3.5)(7.75, 4)
        (5.75,4.75)(6.75, 3.75)(6.5,4.5)(7, 5.25)(7.75, 4.5)(6.25, 5.5)(7, 6.25)(6, 6.5)(8.5, 4.75)(7.5,6.25)(8.25, 7)
        };
        \draw[thick, ->] (9, 46) -- (2, 53);
        \node[anchor=south west] at (2, 53){%
        $\vec d_1$
      };
      \draw[thick, ->] (15, 43) -- (40, 40);
        \node[anchor=south] at (40, 40){%
        $\vec d_2$
      };
       \draw (12,45) circle (0.3cm);
        \end{axis}

    \end{tikzpicture}}
         \caption{Recourse directions}
         \label{fig:loan-example-2-directions}
     \end{subfigure}
     
    \caption{(Left) The training dataset that the loan approval algorithm uses, and (Right) Off manifold $(\vec d_1)$ and On manifold $(\vec d_2)$ directions.}
    \label{fig:loan-example-2}
\end{figure*}

\section{Stepwise Explainable Paths (\Step)}\label{sec:direction-recourse}
We now present our approach for computing recourse directions. 
The overall method is summarized in \Cref{algo:step}. 
We first partition the dataset $\cal X$ into $k$ clusters $\{\cal X_1, \dots, \cal X_k\}$ (using  a standard clustering algorithm). 
For a point of interest $\poi$, we generate a direction $\vec d_c$ towards each cluster $X_c$ using the expression
\begin{align}
	\vec d_c = \sum_{\vec x' \in \cal X_c} (\vec x' - \poi) \alpha(\|\vec x' - \poi\|)\mathbbm{1}(f(\vec x') = 1) \label{eq:step-direction}
\end{align}
where $\alpha:\R_+ \to \R_+$ is some non-negative function and $\|.\|$ is a rotation invariant distance metric. 
We select directions using \Cref{eq:step-direction} (a similar formula is proposed by \citet{sliwinski2019mim}). The intuition behind this equation is as follows: for each point $\vec x'$ in the cluster $\cal X_c$, if $f(\vec x') = 1$, we `move' on the line $(\poi - \vec x'$) a distance of $\alpha(\|\vec x' - \poi\|)$, where $\alpha$ is a decreasing function of $\|\poi - \vec x'\|$. 
Thus, points closer to $\poi$ have a greater effect on $\vec d_c$; similarly, positively classified points that are close to each other will have a greater effect on $\vec d_c$.  
We offer these $k$ directions to the stakeholder, who then returns with a new point $\vec x'$ after following the recourse recommendations. 
The process is repeated until we produce a positively classified point (i.e. a counterfactual) or reach a user-specified maximum number of iterations. \Step's computational complexity is linear in the size of the dataset and is bounded by the maximum number of iterations.
\begin{algorithm}[t]
\caption{Stepwise Explainable Paths (\Step)}
\label{algo:step}
\begin{algorithmic}[1]
\Require Dataset $\cal X$ partitioned into $k$ clusters $\{\cal X_1, \dots, \cal X_k\}$, point of interest $\poi$, model $f$, some non-negatively valued function $\alpha: \R_{\ge 0} \mapsto \R_{\ge 0}$
    
\While{$f(\vec x) = -1$}
	\For{every cluster $c \in [k]$} 
	\Comment{{\tiny \textcolor{black!50!green}{Generate a direction $\vec d_c$ for each $c$}}} 
    \State $\vec d_c \gets \sum_{\vec x' \in \cal X_c} (\vec x' - \vec x) \alpha(\|\vec x' - \vec x\|) \mathbbm{1} (f(\vec x') = 1)$
	\EndFor
    \State Offer the directions $\{d_c\}_{c \in [k]}$ to the stakeholder 
    \State Stakeholder returns an updated point of interest $\vec x^*$
    \State $\vec x \gets \vec x^*$
\EndWhile

\end{algorithmic}
\end{algorithm}

\subsection{An Axiomatically Justified Direction}\label{subsec:step}
We axiomatically derive our choice of direction in \eqref{eq:step-direction}: we identify a direction that uniquely satisfies a set of desirable properties. 
More specifically, for a given point of interest $\vec x \notin \cal X_c$, we believe any reasonable recourse direction, denoted by $\vec d(\vec x, \cal X_c, f)$ should satisfy the following axioms:  
\begin{description}[leftmargin=0cm]
    \item[Shift Invariance (SI)] Let $\cal X_c + \vec b$ denote the dataset resulting from adding the vector $\vec b$ to each point in $\cal X$ and let $f_{\vec b}$ be a shifted model of interest such that $f_{\vec b}(\vec z) = f(\vec z - \vec b)$ for all $\vec z$. Then, $\vec d(\vec x, \cal X_c, f) = \vec d(\vec x + \vec b, \cal X_c + \vec b, f_{\vec b})$.
    \item[Rotation/Reflection Faithfulness (RRF)] Let $A$ be any matrix with $\det(A) \in \{-1, +1\}$ and let $A \cal X_c$ denote the dataset resulting from replacing every point $\vec x^j$ in $\cal X_c$ with $A \vec x^j$. Let $f_A$ denote a rotated model of interest such that $f_A(\vec z) = f(A^{-1} \vec z)$ for all $\vec z$. Then, $A \vec d(\vec x, \cal X_c, f) = \vec d(A\vec x, A\cal X_c, f_A)$.
    \item[Continuity (C)] $\vec d$ is a continuous function of the dataset $\cal X_c$. One can think of the continuity axiom as a notion of recourse \emph{robustness}: small changes to the input PoI will not cause significant changes to the resulting recourse. 
    \item[Data Manifold Symmetry (DMS)] Let $f$ and $g$ be two functions such that $f(\vec x^j) = g(\vec x^j)$ for all points $\vec x^j \in \cal X$. Then, we must have $\vec d(\poi, \cal X_c, f) = \vec d(\poi, \cal X_c, g)$.
    \item[Negative Classification Indifference (NCI)] Let $\vec x' \in \cal X_c$ be a datapoint with $f(\vec x') = -1$. Then, $\vec d(\vec x, \cal X_c, f) = \vec d(\vec x, \cal X_c\setminus \{\vec x'\}, f)$.
    \item[Positive Classification Monotonicity (PCM)] Let $\vec x' \notin \cal X_c$ be a point with $f(\vec x') = 1$ and $ x'_i > x_i$, then $\vec d(\vec x, \cal X_c, f) \le \vec d_i(\vec x, \cal X_c \cup \{\vec x'\}, f)$. Similarly, if $ x'_i < x_i$, then $\vec d_i(\vec x, \cal X_c, f) \ge \vec d_i(\vec x, \cal X_c \cup \{\vec x'\}, f)$.   
    
\end{description}
Our axioms are inspired by \citet{sliwinski2019mim}, who use a similar approach to characterize a family of direction-based explanations (\emph{Monotone Influence Measures}).
Unlike the MIM framework, we remove any dependency on negatively classified points (the Negative Classification Indifference axiom).
Without this change, a naive adaptation of MIM may output directions pointing away from all positively classified points.
Intuitively, a cluster of negative points near positive points may make it impossible to recommend any recourse, as the MIM framework ``shies away'' from negative point clusters (see details in \Cref{apdx:mim}).

The five remaining axioms are fundamental. 
Shift Invariance and Rotation/Reflection Faithfulness ensure the directions depend on the relative locations of the points rather than their absolute values. 
The RRF axiom also generalizes the \emph{feature symmetry} axiom: swapping the coordinates of features $i$ and $j$ does not change the value assigned to them; this is commonly used in the characterization of other model explanation frameworks \citep{Datta2015influence,Datta2016,patel2021high,lundberg2017unexpected,sliwinski2019mim}. 
In addition, these properties ensure the units in which we measure feature values have no effect on the outcome, e.g., measuring income in dollars rather than cents has no effect on the importance of income. 
Continuity ensures that the direction we pick is robust to small changes in the cluster $\cal X_c$. 
Data Manifold Symmetry (DMS) ensures that the direction depends on the model of interest only through points in the dataset. 
In other words, DMS ensures that the model's output on points outside the data manifold do not affect the output direction --- a desirable property in model explanations \citep{Frye2021DataManifold, lundberg2017unexpected,chen2020truetomodel}. 
Positive Classification Monotonicity (PCM) ensures that the direction will point towards regions with a large number of positively classified points. 
In other words, if there is a large number of positively classified points with a high value in some feature $i \in N$ (e.g. bank balance), PCM ensures that the output direction will ask the stakeholder to increase their bank balance.

Any direction which satisfies the above axioms is uniquely given by \Cref{eq:step-direction}: i.e. it is a weighted combination of the directions from the point of interest to every positively classified point in the dataset. The weight given to every point is a function of their distance $\|\vec x' -\poi\|$.
\begin{restatable}{theorem}{stepaxiom}\label{thm:step-axiomatization}
A recourse direction for a point of interest $\poi$ given a dataset $\cal X_c$, a model of interest $f$ and a rotation invariant distance metric $\|.\|$ satisfies SI, RRF, C, DMS, NCI and PCM if and only if it is given by \eqref{eq:step-direction}. 
\end{restatable}
\begin{proof}
For readability, we replace $\cal X_c$ with $\cal X$.
It is easy to see that \eqref{eq:step-direction} satisfies all five axioms so we only show uniqueness.

We assume without loss of generality that $\cal X$ contains no negatively classified points. If $\cal X$ does contain any negatively classified points, we can simply remove them without changing the recourse direction because of Negative Classification Indifference (NCI). Therefore, our goal is to show that any direction which satisfies (SI), (RRF), (C), (DMS), (NCI) and (PCM) is of the form \begin{align*}
    \vec d(\vec x, \cal X, f) = \sum_{\vec x' \in \cal X} \alpha(\|\vec x' - \vec x\|)(\vec x' - \vec x)
\end{align*}
We start off with a useful lemma.

\begin{lemma}\label{lem:axiom-useful}
If any direction $\vec d$ satisfies Rotation and Reflection Faithfulness (RRF) and Positive Classification Monotonicity (PCM), then for any dataset $\cal X$, any datapoint $\poi$, any model of interest $f$ and any positively classified point $\vec y \ne \vec x$, there exists some $a \ge 0$ such that
\begin{align*}
    \vec d(\vec x, \cal X \cup \{\vec y\}, f) - \vec d(\vec x, \cal X, f) = a(\vec y - \vec x)
\end{align*}
\end{lemma}
\begin{proof}
Suppose for contradiction that there is some $\poi$, $\cal X$, $\vec y$ and $f$ such that 
\begin{align*}
    \forall a \ge 0: \vec d(\vec x, \cal X \cup \{y\}, f) - \vec d(\vec x, \cal X, f) \ne a(\vec y - \vec x)
\end{align*}
Let $\vec l = \vec d(\vec x, \cal X \cup \{y\}, f) - \vec d(\vec x, \cal X, f)$. Let $A$ be a rotation matrix such that $(Al)_1 < 0$ and $[A(\vec y - \vec x)]_1 > 0$; such a matrix exists since the two vectors are either linearly independent or $\vec l = -b(\vec y - \vec x)$ where $b \in \mathbb{R}^+$. Since $\vec d$ satisfies Rotation and Reflection Faithfulness (RRF), we have from $(Al)_1 < 0$
\begin{align*}
    \vec d_1(A\poi, A\cal X \cup \{A \vec y\}, f_A) - \vec d_1(A\poi, A\cal X, f_A) < 0
\end{align*}
 This contradicts the Positive Classification Monotonicity (PCM) property since $(A \vec y)_1 > (A \vec x)_1$ and $f_A(A \vec y) = f(\vec y) = 1$.
\end{proof}

Consider a direction $\vec d$ that satisfies the six desired axioms.
We go ahead and show uniqueness via induction on $|\cal X|$. Let $k = 0$, $\cal X = \{\}$. By Shift Invariance (SI), $\vec d(\poi, \{\}, f) = \vec d(\vec 0, \{\}, f_{-\vec x})$. The vector $\vec 0$ and an empty $\cal X$ are invariant under rotation. Therefore, since $\vec d$ satisfies Rotation and Reflection Faithfulness (RRF) and Data Manifold Symmetry (DMS), we must have $\vec d(\vec 0, \{\}, f_{-\vec x}) = \vec 0$, the only vector invariant under rotation.

Let $k = 1$, $\cal X = \{\vec y\}$ where $\vec y \ne \vec x$. Note that any pair of $(\vec x, \vec y)$ can be translated by shift and rotation to any other pair $(\vec x', \vec y')$ with the same distance $(\|\vec y - \vec x\|)$ between them.
This is because the distance metric $(\|.\|)$ is rotation invariant and any distance metric is shift invariant when computing the distance between two points; the shifts cancel each other out. 
Note that after rotation by some matrix $A$, we have $f_A(A \vec y) = f(\vec y)$ and similarly, after shift by some vector $\vec b$, we have $f_{\vec b}(\vec y + \vec b) = f(\vec y)$.
Therefore, the label of the point $y$ does not change after applying Shift Invariance (SI) or Rotation and Reflection Faithfulness (RRF). 
Therefore, by (SI), (RRF) and Lemma \ref{lem:axiom-useful}, we have 
\begin{align*}
    \vec d(\vec x, \cal X, f) = \alpha(\|\vec y - \vec x\|)(\vec y - \vec x)
\end{align*}
where $\alpha$ is a non-negative valued function.

Suppose the hypothesis holds when $|\cal X| \le k$. Consider a dataset $\cal Y$ of size $k+1$. This means $\cal Y$ contains at least two distinct points $\vec y, \vec z \ne \vec x$. We prove our hypothesis for the case where $\vec y$ and $\vec z$ are linearly independent. The case where they are linearly dependent follows from Continuity (C): we can peturb the vectors slightly to make them linearly independent. By Lemma \ref{lem:axiom-useful}, we have 
\begin{align}
    &\vec d(\poi, \cal Y, f) \in A = \{\vec d(\poi, \cal Y \setminus \{\vec y\}, f) + a(\vec y - \vec x)\} \notag\\
    \text{ and } &\vec d(\poi, \cal Y, f) \in B = \{\vec d(\poi, \cal Y \setminus \{\vec z\}, f) + a(\vec z - \vec x)\} \label{eq:step-proof}
\end{align}
By the inductive hypothesis, we have
\begin{align*}
    &\vec d(\poi, \cal Y \setminus \{\vec y\}, f) = \vec d(\poi, \cal Y \setminus \{\vec y, \vec z\}, f) + \alpha(\|\vec z - \vec x\|)(\vec z - \vec x) \\ 
    \text{ and } & \vec d(\poi, \cal Y \setminus \{\vec z\}, f) = \vec d(\poi, \cal Y \setminus \{\vec y, \vec z\}, f) + \alpha(\|\vec y - \vec x\|)(\vec y - \vec x)
\end{align*}
We can use this and the fact that $\vec y - \vec x$ and $\vec z - \vec x$ are linearly independent to combine the two sets in \eqref{eq:step-proof} to get
\begin{align*}
    \vec d(\vec x, \cal Y, f) &= A \cap B = \vec d(\poi, \cal Y \setminus \{\vec y, \vec z\}, f) + \alpha(\|\vec y - \vec x\|)(\vec y - \vec x) + \alpha(\|\vec z - \vec x\|)(\vec z - \vec x) 
\end{align*}
This completes the induction.
\end{proof}
Reducing stakeholder ``travel distance'' to reach a positive classification is a common objective used in algorithmic recourse \citep{Mothilal2020Dice, Karimi2020Mace, Mahajan2019Counterfactual}. 
This distance can be seen as a measure of the effort required from the stakeholder to change their outcome.
We can incorporate this in \Step using the choice of $\alpha$.
Any $\alpha$ function of the form $\alpha(z) = \frac1{z^k}$ with $k \ge 2$ ensures that nearby points are assigned a greater weight than far-off points. 
Different choices of $\alpha$ result in different directions.

\subsection{Diverse and On-Manifold Recourse via Clustering}\label{subsec:step-clustering}

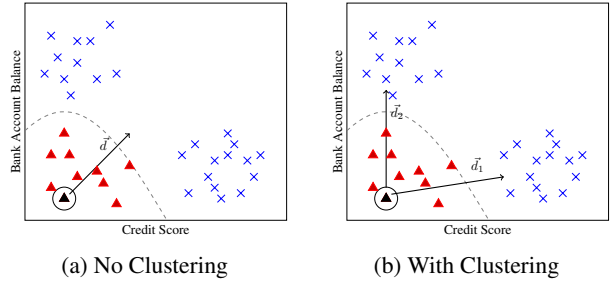
\begin{figure*}
    \centering
    \begin{subfigure}[b]{0.3\textwidth}
         \centering
         \resizebox{0.9\textwidth}{!}{
         \begin{tikzpicture}
        \begin{axis}[
        xmin = 0, 
        xmax = 10, 
        ymin = 0, 
        ymax = 10, 
        ticks = none, 
        xlabel={Credit Score}, 
        ylabel={Bank Account Balance}, 
        ]

        \addplot+[gray, no marks, dashed, domain = 0:5.4] {5*cos(deg((x-1.5)/2.5))};

        \addplot+[only marks, red, mark=triangle*, mark size = 4pt]
        coordinates {(1,1.5)(1.5, 1)(2,2)(2.75, 2.25)(3, 1.7)(1.7, 3)(1.5, 4)(3.5, 0.75)(4, 2.5)(1,3)};


        \addplot+[only marks, black, mark = triangle*, mark size = 4pt, mark options = {fill=black}]
        coordinates {(1.5, 1)};


        \addplot+[only marks, blue, mark=x, mark size = 4pt]
        coordinates {(6, 3)(7, 2)(7.25, 1.5)(6.5, 3.5)(6.25, 1.25)(7.5, 1)(7.75,3)(7.25, 2.75)(8.25,1.25)(8.5,2.75)(8.75, 2)(9, 3)(8, 3.5)(7.75, 4)
        (0.75,6.75)(1.75, 5.75)(1.5,6.5)(2, 7.25)(2.75, 6.5)(1.25, 7.5)(2, 8.25)(1, 8.5)(3.5, 6.75)(2.5,8.25)(3.25, 9)
        };

         \draw[thick, ->] (17, 12) -- (40, 40);
         \node[anchor=south] at (30, 30){$\vec d$};
       \draw (15,10) circle (0.3cm);
        \end{axis}
    \end{tikzpicture}}
        \caption{No Clustering}
         \label{fig:step-diversity}
     \end{subfigure}
     \begin{subfigure}[b]{0.3\textwidth}
         \centering
          \resizebox{0.9\textwidth}{!}{
         \begin{tikzpicture}
        \begin{axis}[
        xmin = 0, 
        xmax = 10, 
        ymin = 0, 
        ymax = 10, 
        ticks = none, 
        xlabel={Credit Score}, 
        ylabel={Bank Account Balance}, 
        ]

        \addplot+[gray, no marks, dashed, domain = 0:5.4] {5*cos(deg((x-1.5)/2.5))};

        \addplot+[only marks, red, mark=triangle*, mark size = 4pt]
        coordinates {(1,1.5)(1.5, 1)(2,2)(2.75, 2.25)(3, 1.7)(1.7, 3)(1.5, 4)(3.5, 0.75)(4, 2.5)(1,3)};


        \addplot+[only marks, black, mark = triangle*, mark size = 4pt, mark options = {fill=black}]
        coordinates {(1.5, 1)};


        \addplot+[only marks, blue, mark=x, mark size = 4pt]
        coordinates {(6, 3)(7, 2)(7.25, 1.5)(6.5, 3.5)(6.25, 1.25)(7.5, 1)(7.75,3)(7.25, 2.75)(8.25,1.25)(8.5,2.75)(8.75, 2)(9, 3)(8, 3.5)(7.75, 4)
        (0.75,6.75)(1.75, 5.75)(1.5,6.5)(2, 7.25)(2.75, 6.5)(1.25, 7.5)(2, 8.25)(1, 8.5)(3.5, 6.75)(2.5,8.25)(3.25, 9)
        };

         \draw[thick, ->] (17, 12) -- (60, 20);
         \node[anchor=south] at (50, 20){$\vec d_1$};
       \draw[thick, ->] (15, 15) -- (15, 60);
         \node[anchor=west] at (15, 50){%
         $\vec d_2$
       };
       \draw (15,10) circle (0.3cm);
        \end{axis}
    \end{tikzpicture}}
        \caption{With Clustering}
         \label{fig:step-issues-clustering}
     \end{subfigure}
    \caption{Directly applying the direction formula \eqref{eq:step-direction} without clustering to find recourse directions may result in undesirable behavior, e.g. excluding a variety of options and picking an off-manifold direction (Left). Clustering resolves this issue (Right).}
    \label{fig:step-issues}
\end{figure*}

Despite its theoretical soundness, utilizing \Step without clustering presents two drawbacks.

{\bf Lack of guaranteed diversity}~~In many cases, changing the $\alpha$ function does not result in a diverse set of directions even when many directions are possible. Consider a slight modification to Example \ref{ex:loan-recourse} (given by Figure \ref{fig:step-issues}) where the class of positive points is separated into two clusters. A recourse algorithm could reasonably output a direction towards either cluster as a potential recourse. However, \Cref{eq:step-direction} outputs a linear combination of these two clusters. Furthermore, it is easy to see that changing the function $\alpha$ will not significantly change the direction.

{\bf Off-manifold directions}~~Since Equation \eqref{eq:step-direction} aggregates directions, it is possible to obtain off-manifold directions (refer to \Cref{fig:step-diversity}); off-manifold regions may have high prediction error.
Without clustering, Equation \eqref{eq:step-direction} aggregates directions obtained for different parts of the dataset, rather than treating different data regions differently. 
We resolve both aforementioned concerns by \emph{clustering} the positively classified data points and compute directions per individual cluster. 
This ensures that we identify different clusters of points and aggregate each cluster in a theoretically sound manner. The impacts are demonstrated in \Cref{fig:step-issues}. 

\subsection{Privacy-Preserving Direction Selection}\label{sec:step-privacy}
One potential concern when using the dataset directly to compute recourse is that \Step could potentially leak sensitive user data (indeed, other model explanation algorithms have been shown to leak private information \cite{Milli2018,shokri2019privacy}). 
With clustering, it may not be possible to offer privacy guarantees since the clustering process itself may not be privacy preserving. 
However, we can show that the \Step distance computation itself is privacy-preserving (see \citet{dwork2014algorithmic} for a formal exposition to differential privacy). 
Briefly, an algorithm $\cal M$ that takes as input a dataset $\cal X$ and outputs a value $\cal M(\cal X)$ is said to be $(\epsilon, \delta)$-differentially private if for all $S \subseteq \text{range}(\cal M)$, we have:
\begin{align*}
    \Pr[\cal M(\cal X) \in S] \le \euler^{\epsilon}\Pr[\cal M(\cal X') \in S] + \delta
\end{align*}
where $\cal X$ and $\cal X'$ are any two datasets that differ by at most one datapoint. 
Differential privacy states that the output of $\cal M$ does not vary much by the removal of any data point.
A simple method to prove that a function is differentially private is to upper bound its {\em sensitivity}, i.e. the change to the function when adding a datapoint. \citet{dwork2014algorithmic} show that adding a finite amount of noise to a bounded sensitivity function guarantees differential privacy. 
A similar guarantee is offered for Shapley-based \citep{Datta2016}, and gradient-based \citep{patel2020explanation} explanations.

\Step (without clustering) can be made differentially private for a specific family of $\alpha$ functions, when the distance metric used is the $\ell_2$ norm. 
More formally, if we assume that $\alpha(z) \le \frac Cz$ for all $z > 0$, then we can bound the sensitivity of \Step.

\begin{restatable}{lemma}{lemstepsensitivity}\label{lem:step-privacy}
When the distance metric used is the $\ell_2$ norm 
and $\alpha(z) \le \frac Cz$ for all $z >0$, 
the global sensitivity (using the $\ell_2$ norm) of the direction given by \eqref{eq:step-direction} is upper bounded by $C$.
\end{restatable}
\begin{proof}
Let the direction output of \eqref{eq:step-direction} (in the absence of clustering) for a particular dataset $\cal X$, a model of interest $f$ and a point of interest $\vec x$ be $\vec d(\vec x, \cal X, f)$.

The global sensitivity of \eqref{eq:step-direction} using the $l_2$ norm  is given by 
\begin{align}
   \Delta_2 \vec d = \max_{\vec x, \cal X, \cal X'} \|\vec d(\vec x, \cal X, f) - \vec d(\vec x, \cal X', f) \|_2 \label{eq:step-gs}
\end{align}
where $\cal X'$ is $\cal X$ with one additional (or one less) datapoint \citep{dwork2014algorithmic}. We can assume without loss of generality that $\cal X'$ contains one additional positively classified datapoint $\vec x'$. If the point is not positively classified, then none of the directions change and the sensitivity is $0$. 
The global sensitivity defined in \eqref{eq:step-gs} reduces to
\begin{align*}
    \Delta_2 \vec d &= \max_{\vec x, \vec x'} \|(\vec x' - \vec x) \alpha(\|\vec x' - \vec x \|_2)\mathbb{I}(f(\vec x') = 1)\|_2 &\le \max_{\vec x, \vec x'} \bigg \|(\vec x' - \vec x) \frac{C}{\|\vec x' - \vec x \|_2}\bigg \|_2 & (\text{by assumptions on $\alpha$ and $f(\vec x')$})\\
    &\le \max_{\vec x, \vec x'} \frac{1}{\|\vec x' - \vec x \|_2} \|C (\vec x' - \vec x)\|_2 = C
\end{align*}
In the first inequality, we assume $||\vec x - \vec x'||_2 > 0$ so we can apply $\alpha(||\vec x - \vec x'||_2) \le \frac{C}{||\vec x - \vec x'||_2}$. If $||\vec x - \vec x'||_2 = 0$, then by definition we must have $\vec x = \vec x'$ which implies $\Delta_2 \vec d = 0$ and the lemma trivially holds.
\end{proof}

Since the direction has bounded sensitivity, classic results from the differential privacy literature tell us that introducing Gaussian noise makes the direction $(\epsilon, \delta)$ differentially private.

\begin{restatable}{theorem}{thmstepprivacy}\label{thm:step-privacy}
When the distance metric used is the $\ell_2$ norm 
and $\alpha(z) \le \frac Cz$, 
the directions output by \eqref{eq:step-direction} can be made $(\epsilon, \delta)$-differentially private by adding Gaussian noise with $0$ mean and standard deviation $\sigma \ge \frac{\beta C^2}{\epsilon}$ where $\beta^2 > 2\log(\frac{1.25}{\delta})$ to all the features.
\end{restatable}
The proof of Theorem \ref{thm:step-privacy} is a direct application of \citet[Theorem 3.22]{dwork2014algorithmic} and is omitted.
Offering multiple recourse directions results in an additive increase in privacy cost.
More specifically, if we provide $k$ directions, and each direction is $(\epsilon, \delta)$ differentially private, then our mechanism is $(k\epsilon, k\delta)$ differentially private \citep{dwork2014algorithmic}.

\section{Empirical Evaluation \& Analysis}\label{sec:expts}\label{sec:experiments}
We compare the performance of \Step and three popular recourse methods---\DICE \citep{Mothilal2020Dice}, \FACE \citep{poyiadzi2020face} and \CCHVAE \citep{pawelczyk2020cchvae}---on three widely used datasets within counterfactual research using three base models. 
We also examine \Step's robustness to noise or ``user-interference'' when following recourse directions.
We provide an overview of our experimental setup here and include full details to support reproducibility in \Cref{apdx:experimental-details}. 
Recall that a counterfactual point for $\poi$ is a terminal point of a recourse path, $\poi_{\CF}$, such that $f(\poi_{\CF}) = 1$. 
  
{\bf Recourse Baselines}~~Given a negatively classified PoI $\poi$, \DICE solves an optimization problem that outputs a diverse set of counterfactuals. 
For each of these counterfactual points $\vec x_{\CF}$, $(\vec x_{\CF} - \poi)$ can be interpreted as a 
direction recommendation for the $\poi$. 
\FACE constructs an undirected graph over the set of datapoints and finds a path from the point of interest $\poi$ to a set of positively classified {\em candidate} points using Djikstra's algorithm \citep{dijkstra1959note}. 
Each edge in the path that connects $\poi$ to $\vec x_{\CF}$ can be thought of as a direction recommendation from $\poi$ to $\vec x_{\CF}$. The backend of \CCHVAE \citep{pawelczyk2020cchvae} is a variational autoencoder (VAE): distances within the latent space surrounding the PoI is used to identify counterfactual points. We use the author's implementation for \DICE and adapt \FACE's and \CCHVAE's implementations from \cite{poyiadzi2020face} and \cite{pawelczyk2021carla}.  

Given a negatively classified PoI $\poi$, \Step and \FACE produce a sequence of points $(\vec x^0, \vec x^1, \dots, \vec x^{\ell})$ where $\vec x^0$ is the original PoI, $\vec x^1$ is the point after following the first direction recommendation by the recourse method, and so on. 
\DICE and \CCHVAE produce two point sequences $(\vec x^0, \vec x^{\ell})$.
We refer to this sequence of points as a {\em recourse path}, and each of the directions can be referred to as a ``step''.

{\bf Datasets and Models}~~We employ three real-world datasets in our cross-comparison analysis: Credit Card Default \citep{yeh2009creditcarddefault}, Give Me Some Credit \citep{credit2011givemesomecredit}, and UCI Adult/Census Income~\citep{Kohavi1996Adult}, described in Table~\ref{tab:datasets}. For each dataset, we train and validate {\bf logistic regression}, {\bf random forest}, and {\bf two-layer DNN} model instances following a 70/15/15 training, validation, and test (recourse-time) data splits. 
Based on a balanced hyperparameter tuning across all recourse methods, we specify a confidence threshold of $0.7$ at test time for each base model to determine whether a PoI is positively classified. Implementation and hyperparameter tuning details are described in \Cref{apdx:models}.

\begin{table}[t!]
\small
\centering
\begin{tabular}{lll|lll}
\hline
    \multicolumn{1}{|l|}{\textbf{Dataset}} & \multicolumn{1}{l|}{\textbf{Type}} & \textbf{Classification task} & \multicolumn{1}{l|}{\textbf{LogReg}} & \multicolumn{1}{l|}{\textbf{RandForest}} & \multicolumn{1}{l|}{\textbf{DNN}} \\ \hline
    \multicolumn{1}{|l|}{Credit Card Default} & \multicolumn{1}{l|}{Financial} & payment default & \multicolumn{1}{l|}{1000} & \multicolumn{1}{l|}{535} & \multicolumn{1}{l|}{559} \\ \hline
    \multicolumn{1}{|l|}{Give Me Some Credit} & \multicolumn{1}{l|}{Financial} & financial distress & \multicolumn{1}{l|}{1000} & \multicolumn{1}{l|}{332} & \multicolumn{1}{l|}{426} \\ \hline
    \multicolumn{1}{|l|}{UCI Adult} & \multicolumn{1}{l|}{Demographic} & income \textgreater{}\$50k & \multicolumn{1}{l|}{1000} & \multicolumn{1}{l|}{1000} & \multicolumn{1}{l|}{1000} \\ \hline
\end{tabular}
\vspace{2pt}
\caption{Dataset descriptions. Values for logistic regression (LogReg), random forest (RandForest) and 2-layer DNN (DNN) reflect average number of negatively classified datapoints in the test set across 10 trials. We limit the number of datapoints used for recourse (from the test split) to 1000.} 
\label{tab:datasets}
\end{table}

{\bf Metrics \& Properties}~~In alignment with counterfactual and recourse literature, we use the following well-established metrics to evaluate the performance of \Step, \DICE, \FACE, and \CCHVAE on each base model and dataset \citep{verma2020counterfactual}. 
In the following definitions, we are given a recourse path $(\poi^0,\vec x^1,\dots,\vec x^\ell)$, where $\poi^0$ is the PoI $\vec x$. A recourse path is \emph{successful} if $f(\vec x^\ell)= 1$, i.e. the recommended path ultimately produced a counterfactual $\vec x^\ell = \vec x_{\CF}$. Our results report the average of these metrics over all PoIs.

{\bf Success} (or \emph{validity} \citep{verma2020counterfactual}) measures the proportion of PoIs with a successful path.

{\bf Average Success} measures the proportion of successful paths generated for a given PoI $\vec x$. 

{\bf $\ell_2$ Distance} is the Euclidean distance between the PoI $\vec x^0$ and the final point $\vec x^{\ell}$, i.e. $\lVert \poi^0 - \vec x^\ell \rVert_2$.
Distance is only computed for successful paths. 
Low distance, or \emph{proximal}, recourse, supports actionability.

{\bf Diversity} is the average Euclidean distances between the counterfactuals of each successful recourse path for a given PoI $\vec x$. 
Diversity is only computed for PoIs with least two successful recourse paths.

\subsection{Categorical Variables and Actionability}\label{subsec:categorical}
We provide rigorous support for categorical variables, including \emph{immutable}, \emph{semi-mutable}, \emph{ordinal} and \emph{unordered}.
Immutable features---ones that cannot be changed, e.g. \texttt{race}---are ignored during recourse. 
Semi-mutable features can only be changed in one direction, e.g., education level can only increase. 
Ordinal features, e.g. an ordered scale like ratings,
are encoded features using a Borda scale (from $1$ to $k$). The remaining categorical features are encoded via one-hot encoding. 
For recourse to be practically relevant, methods should offer \emph{actionable} directions~\citep{ustun2019actionable}: those which do not suggest infeasible changes or changes to immutable features, i.e. that the user can actually execute. \emph{Constraints} are a natural way of preventing recommendations that ask a user to perform an infeasible action, e.g. becoming younger or less educated to secure a loan. \Step encodes such constraints within its distance metric, in a manner similar to \DICE \citep{Mothilal2020Dice}. For each dataset and task, we implement all appropriate encodings and constraints so that the directions produced by \Step satisfy actionability desiderata and requirements. 
Refer to \Cref{apdx:datasets} for encodings by dataset.

\subsection{Comparative Analysis of \Step, \DICE, \FACE, and \CCHVAE}\label{subsec:holistic}

For each PoI, we generate $k=3$ recourse paths, repeat this over $10$ trials, and compute metrics and statistics based on the resulting counterfactuals. 
The comparison of \Step, \DICE, \FACE, and \CCHVAE on all base models are presented in \Cref{tab:all-holistic}. 
We introduce additional metrics and further discussion in \Cref{apdx:additional-holistic-analysis}. 

\begin{table*}[t]
\small
    \centering
    \begin{adjustbox}{max width=\textwidth}
        \begin{tabular}{clrrrrrrrrrrrrr}
        \toprule
                    \multicolumn{1}{c}{}&  & \multicolumn{4}{c}{\textbf{Logistic Regression}} & \multicolumn{4}{c}{\textbf{Random Forest}} & \multicolumn{4}{c}{\textbf{DNN}} \\
                    \cmidrule[0.75pt](lr){3-6}\cmidrule[0.75pt](lr){7-10}\cmidrule[0.75pt](lr){11-14}
                    \textbf{Dataset} & \textbf{Method} &              \rot{Success} & \rot{Avg Success} & \rot{$\ell_2$ Dist.} & \rot{Diversity} &        \rot{Success} & \rot{Avg Success} & \rot{$\ell_2$ Dist.} & \rot{Diversity} &     \rot{Success} & \rot{Avg Success} & \rot{$\ell_2$ Dist.} & \rot{Diversity} & \rot{\textbf{Max Error \%}}\\
        \cmidrule[0.75pt](lr){1-2}\cmidrule[0.75pt](lr){3-6}\cmidrule[0.75pt](lr){7-10}\cmidrule[0.75pt](lr){11-14}\cmidrule[0.75pt](lr){15-15}
        \multirow{4}{*}{\textbf{\thead{Credit\\Card\\Default}}} & StEP &                 1.00 &        0.91 &        7.06 &      2.58 &           1.00 &        0.84 &        3.20 &      0.95 &        1.00 &        1.00 &        5.04 &      1.29 & 8.51\\
                   & DiCE &                 1.00 &        1.00 &       35.28 &     12.93 &           1.00 &        1.00 &       20.34 &      8.35 &        0.99 &        0.99 &       32.47 &     13.02 & 2.96\\
                    & FACE &                 0.54 &        0.54 &        4.46 &      0.78 &           0.51 &        0.51 &        2.75 &      0.88 &        0.45 &        0.45 &        4.38 &      0.93 & 3.16\\
                    & CCHVAE &                 1.00 &        1.00 &        7.88 &      1.17 &           1.00 &        1.00 &        3.22 &      0.29 &        1.00 &        1.00 &        5.11 &      0.34 &5.03\\
        \cmidrule(lr){1-2}\cmidrule(lr){3-6}\cmidrule(lr){7-10}\cmidrule(lr){11-14}\cmidrule(lr){15-15}
        \multirow{4}{*}{\textbf{\thead{Give Me\\Some\\Credit}}} & StEP &                 0.98 &        0.70 &       15.32 &     10.73 &           1.00 &        0.74 &        4.03 &      9.91 &        1.00 &        0.99 &        5.87 &      2.15  &15.94\\
                    & DiCE &                 0.99 &        0.99 &      103.68 &     38.77 &           1.00 &        1.00 &       73.54 &     31.16 &        0.99 &        0.99 &       95.45 &     37.14  &3.33\\
                    & FACE &                 0.98 &        0.98 &        2.97 &      0.57 &           0.96 &        0.96 &        2.79 &      0.63 &        0.93 &        0.93 &        2.80 &      0.59  &0.98\\
                    & CCHVAE &                 0.06 &        0.06 &        2.24 &      0.13 &           1.00 &        1.00 &        1.40 &      0.02 &        1.00 &        1.00 &        4.18 &      0.03  &15.07\\
        \cmidrule(lr){1-2}\cmidrule(lr){3-6}\cmidrule(lr){7-10}\cmidrule(lr){11-14}\cmidrule(lr){15-15}
        \multirow{4}{*}{\textbf{\thead{UCI\\Adult}}} & StEP &                 1.00 &        0.54 &        2.39 &      1.37 &           0.89 &        0.47 &        4.93 &                   2.09 &        1.00 &        0.56 &        2.40 &      1.38  &4.01\\
                    & DiCE &                 1.00 &        1.00 &        6.75 &      1.76 &           1.00 &        1.00 &        7.57 &      1.35 &        1.00 &        1.00 &        7.03 &      1.60  &5.66\\
                    & FACE &                 0.63 &        0.63 &        3.02 &      0.74 &           0.63 &        0.63 &        2.77 &      0.77 &        0.64 &        0.64 &        3.02 &      0.76  &0.83\\
                    & CCHVAE &                 1.00 &        1.00 &        2.57 &      0.46 &           0.93 &        0.92 &        2.59 &      0.44 &        0.99 &        0.99 &        2.77 &      0.51  &8.48\\
            \cmidrule(lr){1-2}\cmidrule(lr){3-6}\cmidrule(lr){7-10}\cmidrule(lr){11-14}\cmidrule(lr){15-15}
                   \multicolumn{2}{r}{\textbf{Max Error \%}}&                 6.70 &        7.25 &        15.94 &      15.07 &           1.89 &        2.95 &        7.86 &      6.49 &        1.50 &        1.50 &        7.72 &      14.83 \\
            \bottomrule
            \end{tabular}
    \end{adjustbox}
    \caption{Comparative analysis results on all datasets and base models. Metrics are computed on scaled data and averaged over $10$ trials. We include max. standard error bounds for each metric by task and across tasks. 
    }
    \label{tab:all-holistic}
\end{table*}

{\bf \Step}~~For all tasks, \Step offers a balance between minimizing distance and maximizing diversity performance. 
\Step's lower distance recourse supports actionability---more proximal counterfactuals may be more actionable to the user---while providing variety in the suggested directions. 
We observe differences between distance and diversity---while maintaining a desirable trade-off between the two metrics---across base models while consistently performing well on success. 
On UCI Adult, \Step exhibits significantly different success compared to average success. 
Given that \Step produces a recourse path for each of its underlying $k$ clusters, average success may be reduced when one or more of these clusters contains many outliers. However, even in this setting, \Step produces successful counterfactuals for the remaining clusters. This supports \Step's robustness to outliers and generalization to the tail of the underlying data distribution.

\Step is also robust to the choice of base model, providing a consistent balance between proximal and diverse counterfactuals with high success. 
Even with a simple base model (logistic regression), \Step balances desirable metrics, while reducing $\ell_2$ distance under more complex yet lightweight models (Random Forest or DNN). 

{\bf \DICE}~~Across tasks, \DICE consistently produces successful recourse paths. Since \DICE selects counterfactuals by solving an optimization problem with a diversity objective, we unsurprisingly observe that the method consistently excels on this metric. 
This performance comes with a trade-off of much higher distance from the given PoI. 
This suggests that \DICE could be useful when highly varied recourse options are desirable, but at the cost of potentially inactionable recourse given the very high distance to the PoI (and therefore very significant changes to be made by the user).


{\bf \FACE}~~When \FACE produces successful recourse, it consistently performs well w.r.t. $\ell_2$ distance---in other words, its counterfactuals are close to the original PoI. In aggregate across datasets, however, \FACE's sensitivity to distance between datapoints in its underlying graph results in unreliable success performance. 
Our results suggest that increasing the nonlinearity of the base model does not assist \FACE in this dimension. Additionally, \FACE exhibits a significant trade-off between low distance between a PoI and its counterfactual(s) and diversity. For $k>1$, \FACE finds the closest counterfactual, an objective which is orthogonal to promoting diversity.  


{\bf \CCHVAE}~~Across tasks, \CCHVAE performs competitively on $\ell_2$ distance, but produces counterfactuals with very low diversity. 
\CCHVAE's objective relies on similarity via distance in the latent space of its VAE, so each of the successful $k$ recourse paths produces a counterfactual which is as close as possible to the original PoI. 
Across nearly all tasks, \CCHVAE excels in terms of success but exhibits a trade-off between proximal successful counterfactuals and diversity.

Despite these strengths, we observe a core weakness of \CCHVAE demonstrated by its surprisingly poor success on Give Me Some Credit under the logistic regression base model. To ensure that generated counterfactuals lie on the data manifold, for a given PoI, \CCHVAE's autoencoder approximates the conditional log likelihood of its mutable features given the immutable features. 
Therefore, for a given immutable value (e.g. a protected class like race or gender), if there are few datapoints with high-confidence positive predictions, \CCHVAE will generate counterfactuals based on low-confidence predictions. 
In this task using logistic regression, only $8\%$ of positively classified training examples had a feature value {\em age} $\leq 59$. This significant feature imbalance means that \CCHVAE generated counterfactuals with a confidence level of $< 0.7$ for these datapoints, demonstrating \CCHVAE's sensitivity to feature imbalance and risk in magnifying biased data.

Unlike with \Step, when \DICE, \FACE, and \CCHVAE are unsuccessful at producing recourse for a given PoI, they fail for all $k$ paths, resulting in equal success and average success.

\begin{table*}[t]
\small
    \centering
    \begin{adjustbox}{max width=\textwidth}
        \begin{tabular}{clrrrrrrrrrrrrr}
        \toprule
                    \multicolumn{1}{c}{}&  & \multicolumn{4}{c}{\textbf{Logistic Regression}} & \multicolumn{4}{c}{\textbf{Random Forest}} & \multicolumn{4}{c}{\textbf{DNN}} \\
                    \cmidrule[0.75pt](lr){3-6}\cmidrule[0.75pt](lr){7-10}\cmidrule[0.75pt](lr){11-14}
                    \textbf{Dataset} & \rot{\textbf{Noise ($\beta$)}} &              \rot{Success} & \rot{Avg Success} & \rot{$\ell_2$ Dist.} & \rot{Diversity} &        \rot{Success} & \rot{Avg Success} & \rot{$\ell_2$ Dist.} & \rot{Diversity} &     \rot{Success} & \rot{Avg Success} & \rot{$\ell_2$ Dist.} & \rot{Diversity} & \rot{\textbf{Max Error \%}}\\
        \cmidrule[0.75pt](lr){1-2}\cmidrule[0.75pt](lr){3-6}\cmidrule[0.75pt](lr){7-10}\cmidrule[0.75pt](lr){11-14}\cmidrule[0.75pt](lr){15-15}
        \multirow{4}{*}{\textbf{\thead{Credit\\Card\\Default}}} & 0.0 &                 1.00 &        0.91 &        7.06 &      2.58 &           1.00 &        0.84 &        3.20 &      0.95 &        1.00 &        1.00 &        5.04 &      1.29 &8.51\\
            & 0.1 &                 1.00 &        0.90 &        7.09 &      2.58 &           1.00 &        0.87 &        3.22 &      0.95 &        1.00 &        1.00 &        5.04 &      1.29 &8.51\\
            & 0.3 &                 1.00 &        0.92 &        7.10 &      2.57 &           1.00 &        0.90 &        3.20 &      0.96 &        1.00 &        1.00 &        5.02 &      1.30 &8.54\\
            & 0.5 &                 1.00 &        0.95 &        7.06 &      2.57 &           1.00 &        0.91 &        3.13 &      0.97 &        1.00 &        1.00 &        5.01 &      1.31 &8.64\\
            \cmidrule(lr){1-2}\cmidrule(lr){3-6}\cmidrule(lr){7-10}\cmidrule(lr){11-14}\cmidrule(lr){15-15}
        \multirow{4}{*}{\textbf{\thead{Give Me\\Some\\Credit}}} & 0.0 &                 0.98 &        0.70 &       15.32 &     10.73 &           1.00 &        0.74 &        4.03 &      9.91 &        1.00 &        0.99 &        5.87 &      2.15 &15.94\\
             & 0.1 &                 0.99 &        0.81 &       13.68 &      9.98 &           1.00 &        1.00 &        3.50 &      2.14 &        1.00 &        0.99 &        5.87 &      2.15 &14.90\\
            & 0.3 &                 0.99 &        0.90 &        9.95 &      7.73 &           1.00 &        1.00 &        2.83 &      1.69 &        0.96 &        0.95 &        3.90 &      2.15 &14.33\\
            & 0.5 &                 0.99 &        0.92 &        8.03 &      6.37 &           1.00 &        1.00 &        2.42 &      1.44 &        0.95 &        0.95 &        3.45 &      2.10 &13.09\\
        \cmidrule(lr){1-2}\cmidrule(lr){3-6}\cmidrule(lr){7-10}\cmidrule(lr){11-14}\cmidrule(lr){15-15}
        \multirow{4}{*}{\textbf{\thead{UCI\\Adult}}} & 0.0 &                 1.00 &        0.54 &        2.39 &      1.37 &           0.89 &        0.47 &        4.93 &      2.09 &        1.00 &        0.56 &        2.40 &      1.38 &4.01\\
            & 0.1 &                 1.00 &        0.53 &        2.35 &      1.36 &           0.89 &        0.47 &        4.93 &      2.07 &        1.00 &        0.55 &        2.35 &      1.37 &3.84\\
            & 0.3 &                 1.00 &        0.53 &        2.34 &      1.36 &           0.89 &        0.47 &        4.91 &      2.07 &        1.00 &        0.55 &        2.34 &      1.37 &3.94\\
            & 0.5 &                 1.00 &        0.53 &        2.32 &      1.35 &           0.89 &        0.47 &        4.86 &      2.06 &        1.00 &        0.55 &        2.31 &      1.36&3.96\\
            \cmidrule(lr){1-2}\cmidrule(lr){3-6}\cmidrule(lr){7-10}\cmidrule(lr){11-14}\cmidrule(lr){15-15}
                   \multicolumn{2}{r}{\textbf{Max Error \%}}&                  1.06 &        7.25 &        15.94 &      12.41 &           1.64 &        2.95 &        11.14 &      14.24 &        0.72 &        0.80 &        9.41 &      14.90 \\
            \bottomrule
            \end{tabular}
    \end{adjustbox}
    \caption{User-interference experiment results. Metrics are computed on scaled data and averaged over $10$ trials. Maximum standard error bounds for each metric by task and across tasks are included.}
   \label{tab:all-noise}
\end{table*}

\subsection{Practical Robustness Considerations for \Step} \label{subsec:robustness}
 
\paragraph{Clustering.} We evaluate \Step's sensitivity to the choice of underlying clusters across all tasks and base models. 
Firstly, using off-the-shelf \emph{k-means}, we vary $k$ between $\{1, \dots, 6\}$. 
Our empirical results in \Cref{apdx:clustering} show that \Step is robust to both the number and relative size of clusters when using $k$-means. 
We also evaluate \Step under {\bf random clustering}, by assigning each point to a random cluster value from $\{1, \dots, k\}$. 
Our metrics of interest slightly improve as $k$ increases. 
We hypothesize that this improvement stems from the finer-grained dataset partitioning, which allows \Step to produce a path (within the given maximum number of iterations) from the PoI to at least one cluster. 
Even under random clustering, \Step performs well on most metrics. 
Given the strict categorical constraints in the UCI Adult task, the center of each random cluster being distributed uniformly across the entire dataset results in some clusters being inherently less reachable for many PoIs.

\paragraph{User interference.}  
Algorithmic recourse models commonly assume that users exactly follow suggested actions. 
In \Cref{apdx:step-user-interference}, we relax this assumption by introducing a noise parameter $\beta$ as a proxy for how much a user deviates from the prescribed direction. 
We construct a noise vector where each dimension that represents a continuous feature in the dataset is independently sampled from the standard normal distribution and the remaining dimensions are zero. We scale this noise vector to magnitude $\beta \times \lVert \vec d \rVert$, where $\vec d$ is the original recourse vector and $\beta \in \mathbbm{R}_{\geq 0}^m$. The noise vector is then added to $\vec d$ as the next suggested action.

The results presented in \Cref{tab:all-noise} demonstrate \Step's robustness to user-interference and other noise. 
The \emph{improvements} in performance w.r.t. success, avg success, and $\ell_2$ distance in many cases can be interpreted as noise providing a benefit similar to small amounts of stochasticity in gradient methods (e.g. SGD, momentum), helping \Step move out of a local minima. 
We include additional results and discussion on this task in \Cref{apdx:step-user-interference}.
This case study provides considerable empirical evidence to \Step's robustness to deviation from the suggested recourse.

\section{Discussion and Conclusion} 
We introduce \Step, a data-driven method for direction-based algorithmic recourse that does not depend on the underlying model or knowledge of its underlying causal relations. 
We show that the directions computed by \Step uniquely satisfy a set of desirable properties, which can be made differentially private under some mild assumptions. 
We empirically demonstrate \Step's ability to produce actionable and diverse recourse, its robustness to user-interference, and its practical utility. 

In Section \ref{subsec:step-clustering}, we discuss the limitations of \Step without clustering and highlight the necessity of a clustering-based approach. Modulating $k$ controls the number of clusters (and, in turn, number of recourse directions produced); a clustering with too fine of a granularity may result in unsuccessful recourse paths for some of the clusters.
While standard clustering methods perform well in practice, a thorough analysis of their effects on \Step could provide useful insights in practical deployment settings.

An additional limitation of \Step is its dependence on how we measure distances between datapoints, as with \DICE, \FACE and Wachter's method. 
Defining new encoding schemas for categorical variables and investigating different notions of distance in the use of \Step and other recourse methods is an interesting area of future work.
Finally, \Step and other recourse methods raise questions regarding group fairness, e.g. whether \Step offers similar performance for users from different protected groups.  
\paragraph{Acknowledgements.}
The authors thank Neel Patel for early discussions. Perello and Zick are supported by Army Research Lab DEVCOM Data and Analysis Center - Contract W911QX23D0009.

\bibliographystyle{tmlr}
\bibliography{abb,literature}

\newpage
\appendix

\onecolumn

\section{Issue with the Naive Adaptation of MIM}\label{apdx:mim}
MIM (proposed by \citet{sliwinski2019mim}) is a data driven approach to explain the outcome of point $\vec x$. Intuitively, it is a direction that moves towards points of the same outcome and moves away from points of the opposite outcome. It is given by:
\begin{align*}
    \vec d(\poi, \cal X, f) = \sum_{\vec x' \in \cal X} (\vec x' - \poi) \alpha(\|\vec x' - \poi \|)\mathbbm{1}(f(\vec x') = f(\vec x))
\end{align*}
where $\alpha$ is a non-negative valued function and $\mathbbm{1}(.)$ takes the value $1$ if the input condition is true and $0$ otherwise.
We can naively adapt MIM for recourse by flipping the direction, moving towards positively classified points and away from negatively classified points. This would give us the direction:
\begin{align*}
    \vec d(\poi, \cal X, f) = \sum_{\vec x' \in \cal X} (\vec x' - \poi) \alpha(\|\vec x' - \poi \|)\mathbbm{1}(f(\vec x') \ne f(\vec x))
\end{align*}
This direction, unfortunately, can lead to bad recourse recommendations. 
To see why, consider Example \ref{ex:loan-recourse} with a different dataset and point of interest (given in Figure \ref{fig:mim-issue}). When nearby points are given a higher weight than far away points, the direction output by MIM could be poor, pointing away from all the positively classified points. This is because when nearby points are given a higher weight than far away points, the push away from negatively classified points can be stronger than the pull towards positively classified points. This is indeed what happens in Figure \ref{fig:mim-issue}.

\begin{figure}
    \centering
         \begin{tikzpicture}
        \begin{axis}[ 
        xmin = 0, 
        xmax = 10, 
        ymin = 0, 
        ymax = 10, 
        ticks = none, 
        xlabel={Credit Score}, 
        ylabel={Bank Account Balance}, 
        legend pos=north east]
        
        \addplot+[gray, no marks, dashed, domain = 0:6] table[row sep = crcr]{5 0 \\ 5 10 \\};
        \addlegendentry{Boundary};
        \addplot+[only marks, red, mark=triangle*, mark size = 4pt, mark options = {fill=red}]
        coordinates {(2,1.5)(2.5, 1)(3,2)(3.75, 2.25)(4, 3)(2.7, 3)(3, 4)(3.5, 4.75)(2.75, 5)(3, 6)(4, 5.5)(2, 5.75)};
        
        \addlegendentry{Rejection}
        
        \addplot+[only marks, black, mark = triangle*, mark size = 4pt, mark options = {fill=black}]
        coordinates {(2, 4)};

        \addlegendentry{POI}
        
        \addplot+[only marks, blue, mark=x, mark size = 4pt]
        coordinates {(7,1.5)(7.5, 1)(7,2)(7.75, 2.25)(8, 3)(6.7, 3)(7, 4)(6.75, 5)(7, 6)(8, 5.5)(6, 5.75)
        (6.5,5.5)
        };
        \addlegendentry{Acceptance}

        \draw[thick, ->] (18, 40) -- (5, 40);
        \node[anchor= north] at (10, 40){%
        $\vec d_{\textit{MIM}}$
      };
       \draw (20,40) circle (0.3cm);
        \end{axis}
        
    \end{tikzpicture}

    \caption{An example where the MIM direction points away from all the positively classified points.}
    \label{fig:mim-issue}
\end{figure}
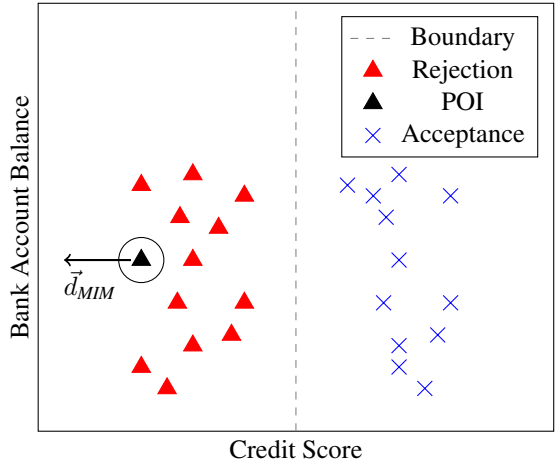

\section{Additional Empirical Investigation} \label{apdx:additional-holistic-analysis}

We compare \Step with three methods, \DICE, \FACE, and \CCHVAE by generating $k=3$ recourse instructions for each negatively classified datapoint in the test set for each dataset. 

With \Step, we first partition the positively labeled training data into $3$ clusters, and then for each PoI in the test set, we produce a direction for each of the $3$ clusters. In these comparative experiments, we use sci-kit learn's default k-means implementation without any tuning, and we assume the the stakeholder follows the provided direction exactly.

\subsection{Baseline Recourse Methods}\label{apdx:baselines}
We compare \Step to \DICE \citep{Mothilal2020Dice}, \FACE \citep{poyiadzi2020face}, and \CCHVAE \cite{pawelczyk2020cchvae}. \DICE outputs a diverse set of points, \FACE outputs a set of paths that terminate at a positively classified point, and \CCHVAE outputs points in the latent neighborhood of the PoI using an autoencoder. 

Given a point $\vec x$ such that $f(\vec x) = -1$, \DICE uses determinantal point processes and solves the following optimization problem to output a diverse set of $m$ counterfactual explanations $\{\vec c_1, \vec c_2, \dots, \vec c_m\}$:
\begin{align*}
    \argmin_{\vec c_1, \vec c_2, \dots, \vec c_m} \frac1m \sum_{i = 1}^m \textit{loss}(f(\vec c_i), 1) + \frac{\lambda_1}m \sum_{i = 1}^m \textit{dist}(\vec c_i, \vec x) \\ - \lambda_2 \text{dpp\_diversity}(\vec c_1, \vec c_2, \dots, \vec c_m) 
\end{align*}
Each of the output points $\vec c_j$ can be seen as a recommendation for the direction $\vec c_j - \vec x$. Due to the structure of the optimization problem, none of these output points are guaranteed to be positively classified. We run \DICE with the default hyperparameter settings. 

FACE constructs an undirected graph over the set of data points and finds a path from the point of interest to a set of positively classified {\em candidate} points using Djikstra's algorithm \citep{dijkstra1959note}. All the results generated in this paper use a distance threshold of $3$ and the max length of a path is capped at $50$ data points. 
 
\CCHVAE utilizes a variational autoencoder to construct a low-dimensional neighborhood within a radius around the PoI and searches for counterfactuals within it. At each iteration the radius to search around gets expanded by a given step distance hyperparameter. We run \CCHVAE with a step distance of $1$ and set the max number of iteration to $50$. For the remaining hyperparameters, we use the default given by \cite{pawelczyk2021carla}. We modify \citeauthor{pawelczyk2021carla}'s implementation for \CCHVAE to support non-binary categorical features and $k>1$ counterfactuals.

\subsection{Additional Comparative Analysis}\label{apdx:additional-comparative-analysis}
We define three appropriate metrics, {\em path length}, {\em path steps}, and {\em proximal diversity} motivated by our direction-based algorithm.

\noindent{\bf Path length} is the sum of Euclidean distances between the steps of a recourse path. 
A shorter path length requires stakeholders to work less to change their outcome.
Formally, the path length is given by $\sum_{i=1}^{\ell} \lVert \vec x^{i} - \vec x^{i-1}\rVert_2$.
Path lengths are only computed for successful paths.

\noindent {\bf Path steps} corresponds to the number of steps taken in the path, given by $\ell - 1$. When the number of steps is $1$ (as is with \DICE and \CCHVAE), distance and path length are the same. 

\noindent {\bf Proximal Diversity} (Prox. Diver.) is the total Euclidean distances between the counterfactual points of each successful recourse path for a given PoI $\vec x$ normalized by distance from the furthest terminal point to the PoI, $$\frac{1}{\max_{i \in 1,\ldots,k} \lVert \poi^0 - \vec x^\ell_i \rVert_2}\sum_{i=1}^k\sum_{j=i}^k\lVert \vec x^\ell_i - \vec x^\ell_j \rVert_2.$$
To provide diverse recourse one can simply provide counterfactual points significantly distant from each other and the PoI. This makes recourse less actionable and the metric value less interpretable, therefore we normalize by max distance (proximity) between the PoI and its counterfactuals to scale this.
Proximal diversity is only computed for PoIs where at least two recourse paths are successful.

The results for these new metrics are presented in Table \ref{tab:all-holistic-additional}. Proximal diversity reflects our observations on distance and diversity in \Cref{subsec:holistic}. That is, \Step has diverse and close in distance recourse and \DICE has highly diverse and far in distance recourse. Proximal diversity weighs these two aspects, resulting in similar performance between the two methods. Given that \DICE and \CCHVAE do not output points between the PoI and counterfactual, path steps was always one step.

\begin{table*}[t]
    \centering
    \begin{adjustbox}{max width=\textwidth}
        \begin{tabular}{clrrrrrrrrrrrrr}
        \toprule
                    \multicolumn{1}{c}{}&  & \multicolumn{4}{c}{\textbf{Logistic Regression}} & \multicolumn{4}{c}{\textbf{Random Forest}} & \multicolumn{4}{c}{\textbf{DNN}} \\
                    \cmidrule[0.75pt](lr){3-6}\cmidrule[0.75pt](lr){7-10}\cmidrule[0.75pt](lr){11-14}
                    \textbf{Dataset} & \textbf{Method} &              \rot{Success} & \rot{Path Length} & \rot{Path Steps} & \rot{Prox. Diver.} &        \rot{Success} & \rot{Path Length} & \rot{Path Steps} & \rot{Prox. Diver.} &     \rot{Success} & \rot{Path Length} & \rot{Path Steps} & \rot{Prox. Diver.} & \rot{\textbf{Max Error \%}}\\
        \cmidrule[0.75pt](lr){1-2}\cmidrule[0.75pt](lr){3-6}\cmidrule[0.75pt](lr){7-10}\cmidrule[0.75pt](lr){11-14}\cmidrule[0.75pt](lr){15-15}
        \multirow{4}{*}{\textbf{\thead{Credit\\Card\\Default}}} & StEP &                 1.00 &        7.23 &       7.42 &      2.49 &           1.00 &        3.70 &       3.75 &      2.22 &        1.00 &        5.04 &       5.13 &      2.09 & 6.88\\
            & DiCE &                 1.00 &       35.28 &       1.00 &      2.43 &           1.00 &       20.34 &       1.00 &      2.23 &        0.99 &       32.47 &       1.00 &      2.64 & 2.96\\
            & FACE &                 0.54 &        5.61 &       2.18 &      2.02 &           0.51 &        3.68 &       1.56 &      2.49 &        0.45 &        5.04 &       2.06 &      2.16 & 1.75\\
            & CCHVAE &                 1.00 &        7.88 &       1.00 &      0.96 &           1.00 &        3.22 &       1.00 &      0.76 &        1.00 &        5.11 &       1.00 &      0.57  & 3.53\\
        \cmidrule(lr){1-2}\cmidrule(lr){3-6}\cmidrule(lr){7-10}\cmidrule(lr){11-14}\cmidrule(lr){15-15}
        \multirow{4}{*}{\textbf{\thead{Give Me\\Some\\Credit}}} & StEP &                 0.98 &       15.86 &      15.86 &      2.41 &           1.00 &        4.12 &       4.12 &      2.09 &        1.00 &        5.87 &       5.87 &      2.19 & 15.42\\
            & DiCE &                 0.99 &      103.68 &       1.00 &      2.34 &           1.00 &       73.54 &       1.00 &      2.28 &        0.99 &       95.45 &       1.00 &      2.38 & 3.33\\
            & FACE &                 0.98 &        2.72 &       1.07 &      1.92 &           0.96 &        2.80 &       1.01 &      2.01 &        0.93 &        2.81 &       1.01 &      1.91 & 1.13\\
            & CCHVAE &                 0.06 &        2.24 &       1.00 &      0.13 &           1.00 &        1.40 &       1.00 &      0.15 &        1.00 &        4.18 &       1.00 &      0.13 & 15.07\\
        \cmidrule(lr){1-2}\cmidrule(lr){3-6}\cmidrule(lr){7-10}\cmidrule(lr){11-14}\cmidrule(lr){15-15}
        \multirow{4}{*}{\textbf{\thead{UCI\\Adult}}} & StEP &                 1.00 &        2.70 &       2.70 &      2.48 &           0.89 &        5.35 &       5.46 &      2.46 &        1.00 &        2.57 &       2.65 &      2.47 & 3.98\\
            & DiCE &                 1.00 &        6.75 &       1.00 &      1.65 &           1.00 &        7.57 &       1.00 &      1.20 &        1.00 &        7.03 &       1.00 &      1.43 & 3.98\\
            & FACE &                 0.63 &        4.29 &       1.87 &      2.14 &           0.63 &        3.83 &       1.72 &      2.41 &        0.64 &        4.18 &       1.84 &      2.16 & 0.83\\
            & CCHVAE &                 1.00 &        2.57 &       1.00 &      1.24 &           0.93 &        2.69 &       1.00 &      1.24 &        0.99 &        2.83 &       1.00 &      1.35 & 5.19\\
            \cmidrule(lr){1-2}\cmidrule(lr){3-6}\cmidrule(lr){7-10}\cmidrule(lr){11-14}\cmidrule(lr){15-15}
                   \multicolumn{2}{r}{\textbf{Max Error \%}}&                 6.70 &        15.42 &        15.42 &      15.07 &           2.21 &        7.58 &        7.58 &      5.13 &        1.50 &        7.72 &        7.72 &      6.88 \\
            \bottomrule
            \end{tabular}
    \end{adjustbox}
    \caption{Comparative analysis results on all datasets and base models. Metrics are computed on scaled data and averaged over $10$ trials. The maximum standard error bounds for each metric by task (row) and across tasks (column) are included.
    }
    \label{tab:all-holistic-additional}
\end{table*}

\subsection{Discussion of \DICE}\label{apdx:dice-discuss}
Since \DICE is not a direction-based method, we generate $k$ counterfactual points and interpret the difference between the original PoI $\poi$ and counterfactual $\vec x_{\CF}$ as a proxy for the path to take. In our experiments, \DICE almost always generates a valid counterfactual explanation but this counterfactual is usually much farther from the PoI than \Step on average, which requires the user to make very large changes at each step of recourse. Furthermore, sometimes the counterfactual can be away from the data manifold.

\subsection{Discussion of \FACE}\label{apdx:face-discuss}
\FACE is a direction-based recourse method which finds the shortest path from each PoI $\poi$ to $k$ positively classified candidate points that exist within the graph (i.e.\ counterfactuals). The existence of edges between points in a graph are determined using a {\em distance threshold} parameter. If two points have a distance less than the distance threshold, an edge is generated; otherwise, no edge spans the two points. Intuitively, the distance threshold determines the size of the step users are required to take when following \FACE's recourse. 

Finding an appropriate distance threshold for \FACE is a challenging and highly task-dependent requirement. If the distance threshold is too small, the graph becomes sparse and typically produces few to no successful recourse paths. On the other hand, if the distance threshold is very large, the graph becomes dense and recourse generates trivial paths between PoIs and candidate points consisting of a single edge. In this case, \FACE produces a brittle recourse. 

One possible option for the distance threshold is to set it equal to the step-size of \Step (equal to $1$) to ensure a fair comparison of the two methods, which results in very few successful paths being generated (in some tasks, with a success rate as low as $0$). When increasing the distance threshold between $2$ to $3$, \FACE begins non-trivial paths across our tasks. All the results generated in this paper are using the distance threshold $3$. 

\subsection{Discussion of \CCHVAE}\label{apdx:cchvae-discuss}
\CCHVAE uses both the encoder and decoder from a variational autoencoder trained on the given dataset. The encoder transforms data points into a low-dimensional representation and the decoder reconstructs latent representations to the original dimension. During it's search, \CCHVAE encodes the PoI and samples points around the hyper-sphere around it. These points are then decoded and tested on the base model to determine if they are counterfactuals. In each iteration, \CCHVAE expands the hyper-sphere it samples around by a given step distance hyperparamter. To find $k$ counterfactuals, we continue the algorithm until $k$ counterfactuals are found via sampling.

In most tasks, this iterative process of expanding the search radius step-by-step aids \CCHVAE in producing low distance counterfactuals with high success.
However, the encoding and decoding steps lead \CCHVAE to be sensitive to strong relationships between mutable and immutable features, leading to failures in certain tasks as described in \Cref{subsec:holistic}. The choice of number of iterations, step distance, or number of $k$ counterfactuals did not affect the poor success on Give Me Some Credit under logistic regression. 
We recommend careful dataset selection when using \CCHVAE, as the relation between immutable and mutable features can be and is often discriminatory.

\begin{table*}[t]
    \centering
    \begin{adjustbox}{max width=\textwidth}
        \begin{tabular}{clrrrrrrrrrrrrr}
        \toprule
                    \multicolumn{1}{c}{}&  & \multicolumn{4}{c}{\textbf{Logistic Regression}} & \multicolumn{4}{c}{\textbf{Random Forest}} & \multicolumn{4}{c}{\textbf{DNN}} \\
                    \cmidrule[0.75pt](lr){3-6}\cmidrule[0.75pt](lr){7-10}\cmidrule[0.75pt](lr){11-14}
                    \textbf{Dataset} & \rot{\textbf{Noise ($\beta$)}} &              \rot{Success} & \rot{Path Length} & \rot{Path Steps} & \rot{Prox. Diver.} &        \rot{Success} & \rot{Path Length} & \rot{Path Steps} & \rot{Prox. Diver.} &     \rot{Success} & \rot{Path Length} & \rot{Path Steps} & \rot{Prox. Diver.} & \rot{\textbf{Max Error \%}}\\
        \cmidrule[0.75pt](lr){1-2}\cmidrule[0.75pt](lr){3-6}\cmidrule[0.75pt](lr){7-10}\cmidrule[0.75pt](lr){11-14}\cmidrule[0.75pt](lr){15-15}
        \multirow{4}{*}{\textbf{\thead{Credit\\Card\\Default}}} & 0.0 &                 1.00 &        7.23 &       7.42 &      2.49 &           1.00 &        3.70 &       3.75 &      2.22 &        1.00 &        5.04 &       5.13 &      2.09 &6.88\\
             & 0.1 &                 1.00 &        7.43 &       7.73 &      2.49 &           1.00 &        4.26 &       4.31 &      2.22 &        1.00 &        5.06 &       5.15 &      2.09 &6.82\\
            & 0.3 &                 1.00 &        7.91 &       8.34 &      2.49 &           1.00 &        4.75 &       4.85 &      2.25 &        1.00 &        5.19 &       5.33 &      2.09 &6.66\\
            & 0.5 &                 1.00 &        8.36 &       8.93 &      2.48 &           1.00 &        4.73 &       4.92 &      2.28 &        1.00 &        5.44 &       5.69 &      2.11 &6.41\\
            \cmidrule(lr){1-2}\cmidrule(lr){3-6}\cmidrule(lr){7-10}\cmidrule(lr){11-14}\cmidrule(lr){15-15}
        \multirow{4}{*}{\textbf{\thead{Give Me\\Some\\Credit}}} & 0.0 &                 0.98 &       15.86 &      15.86 &      2.41 &           1.00 &        4.12 &       4.12 &      2.09 &        1.00 &        5.87 &       5.87 &      2.19 &15.42\\
            & 0.1 &                 0.99 &       16.26 &      16.27 &      2.42 &           1.00 &        3.52 &       3.53 &      2.77 &        1.00 &        5.89 &       5.89 &      2.20 &11.20\\
            & 0.3 &                 0.99 &       12.59 &      12.64 &      2.48 &           1.00 &        2.92 &       2.93 &      2.86 &        0.96 &        4.03 &       4.05 &      2.24 &9.71\\
            & 0.5 &                 0.99 &       10.45 &      10.57 &      2.54 &           1.00 &        2.60 &       2.62 &      2.94 &        0.95 &        3.77 &       3.81 &      2.29 &8.52\\
        \cmidrule(lr){1-2}\cmidrule(lr){3-6}\cmidrule(lr){7-10}\cmidrule(lr){11-14}\cmidrule(lr){15-15}
        \multirow{4}{*}{\textbf{\thead{UCI\\Adult}}} & 0.0 &                 1.00 &        2.70 &       2.70 &      2.48 &           0.89 &        5.35 &       5.46 &      2.46 &        1.00 &        2.57 &       2.65 &      2.47 &3.98\\
            & 0.1 &                 1.00 &        2.64 &       2.94 &      2.49 &           0.89 &        5.81 &       6.34 &      2.47 &        1.00 &        2.46 &       2.82 &      2.47 &4.76\\
            & 0.3 &                 1.00 &        2.63 &       3.01 &      2.49 &           0.89 &        5.79 &       6.47 &      2.46 &        1.00 &        2.44 &       2.88 &      2.47 &4.72\\
            & 0.5 &                 1.00 &        2.62 &       3.16 &      2.48 &           0.89 &        5.70 &       6.73 &      2.45 &        1.00 &        2.43 &       3.02 &      2.47 &4.77\\
            \cmidrule(lr){1-2}\cmidrule(lr){3-6}\cmidrule(lr){7-10}\cmidrule(lr){11-14}\cmidrule(lr){15-15}
                   \multicolumn{2}{r}{\textbf{Max Error \%}}&                  1.06 &        15.42 &        15.42 &      5.20 &           1.64 &        11.20 &        11.20 &      2.35 &        0.72 &        9.41 &        9.41 &      6.88 \\
            \bottomrule
            \end{tabular}
    \end{adjustbox}
    \caption{User-interference analysis results on all datasets and base models. Metrics are computed on scaled data and averaged over $10$ trials. The maximum standard error bounds for each metric by task (row) and across tasks (column) are included. ``Prox. Diver.'' is shorthand for proximal diversity.}
   \label{tab:all-noise-appendix}
\end{table*}
\subsection{The Effects of Clustering on \Step}
\label{apdx:clustering}

\begin{table*}[t]
    \centering
    \begin{adjustbox}{max width=\textwidth}
        \begin{tabular}{clrrrrrrrrrrrrr}
        \toprule
                    \multicolumn{1}{c}{}&  & \multicolumn{4}{c}{\textbf{Logistic Regression}} & \multicolumn{4}{c}{\textbf{Random Forest}} & \multicolumn{4}{c}{\textbf{DNN}} \\
                    \cmidrule[0.75pt](lr){3-6}\cmidrule[0.75pt](lr){7-10}\cmidrule[0.75pt](lr){11-14}
                    \textbf{Dataset} & \textbf{$k$} &              \rot{Success} & \rot{Avg Success} & \rot{$\ell_2$ Dist.} & \rot{Diversity} &        \rot{Success} & \rot{Avg Success} & \rot{$\ell_2$ Dist.} & \rot{Diversity} &     \rot{Success} & \rot{Avg Success} & \rot{$\ell_2$ Dist.} & \rot{Diversity}& \rot{\textbf{Max Error \%}}\\
        \cmidrule[0.75pt](lr){1-2}\cmidrule[0.75pt](lr){3-6}\cmidrule[0.75pt](lr){7-10}\cmidrule[0.75pt](lr){11-14}\cmidrule[0.75pt](lr){15-15}
        \multirow{4}{*}{\textbf{\thead{Credit\\Card\\Default}}} & 1 &                 0.90 &        0.90 &        5.03 &      0.00 &           1.00 &        1.00 &        2.99 &      0.00 &        1.00 &        1.00 &        4.65 &      0.00 & 1.27\\
                & 2 &                 0.92 &        0.87 &        6.17 &      2.01 &           1.00 &        1.00 &        3.10 &      0.77 &        1.00 &        1.00 &        4.95 &      1.10 & 1.46 \\
                & 3 &                 1.00 &        0.91 &        7.06 &      2.58 &           1.00 &        0.84 &        3.20 &      0.95 &        1.00 &        1.00 &        5.04 &      1.29 &8.51 \\
                & 4 &                 1.00 &        0.93 &        6.96 &      4.48 &           1.00 &        0.86 &        3.30 &      1.78 &        1.00 &        1.00 &        5.06 &      2.38 &6.54 \\
                & 5 &                 1.00 &        0.95 &        7.06 &      4.96 &           1.00 &        0.88 &        3.65 &      2.52 &        1.00 &        1.00 &        5.19 &      2.83 &11.13 \\
                & 6 &                 1.00 &        0.95 &        7.22 &      5.50 &           1.00 &        0.90 &        3.44 &      2.40 &        1.00 &        1.00 &        5.29 &      3.37 & 8.56\\
            \cmidrule(lr){1-2}\cmidrule(lr){3-6}\cmidrule(lr){7-10}\cmidrule(lr){11-14}\cmidrule(lr){15-15}
        \multirow{4}{*}{\textbf{\thead{Give Me\\Some\\Credit}}} & 1 &                 0.58 &        0.58 &        2.37 &      0.00 &           1.00 &        1.00 &        2.83 &      0.00 &        1.00 &        1.00 &        4.11 &      0.00 & 8.67\\
                & 2 &                 0.90 &        0.70 &        8.78 &      6.58 &           1.00 &        0.73 &        4.12 &     10.28 &        1.00 &        1.00 &        4.17 &      0.27 & 31.71\\
                & 3 &                 0.98 &        0.70 &       15.32 &     10.73 &           1.00 &        0.74 &        4.03 &      9.91 &        1.00 &        0.99 &        5.87 &      2.15 &15.94\\
                & 4 &                 1.00 &        0.75 &       15.27 &     18.64 &           1.00 &        0.79 &        6.56 &     16.35 &        1.00 &        0.98 &        6.26 &      4.42 & 10.51\\
                & 5 &                 1.00 &        0.73 &       15.60 &     19.85 &           1.00 &        0.81 &        7.21 &     17.57 &        1.00 &        0.98 &        6.11 &      4.66 & 11.00\\
                & 6 &                 1.00 &        0.73 &       14.73 &     19.11 &           1.00 &        0.82 &        7.58 &     18.57 &        1.00 &        0.98 &        6.23 &      5.01 & 11.24\\
        \cmidrule(lr){1-2}\cmidrule(lr){3-6}\cmidrule(lr){7-10}\cmidrule(lr){11-14}\cmidrule(lr){15-15}
        \multirow{4}{*}{\textbf{\thead{UCI\\Adult}}} & 1 &                 0.24 &        0.24 &        1.89 &      0.00 &           0.23 &        0.23 &        2.09 &      0.00 &        0.35 &        0.35 &        1.86 &      0.00 &3.21\\
                & 2 &                 1.00 &        0.62 &        1.94 &      0.86 &           0.89 &        0.56 &        4.53 &      1.98 &        1.00 &        0.67 &        2.06 &      0.85 & 4.50\\
                & 3 &                 1.00 &        0.54 &        2.39 &      1.37 &           0.89 &        0.47 &        4.93 &      2.09 &        1.00 &        0.56 &        2.40 &      1.38 &4.01\\
                & 4 &                 1.00 &        0.48 &        2.35 &      2.12 &           0.89 &        0.39 &        4.91 &      3.15 &        1.00 &        0.63 &        2.37 &      2.11 & 4.03\\
                & 5 &                 1.00 &        0.58 &        2.28 &      2.09 &           0.90 &        0.44 &        4.20 &      3.05 &        1.00 &        0.63 &        2.33 &      2.07 & 5.34\\
                & 6 &                 1.00 &        0.56 &        2.33 &      2.16 &           0.91 &        0.44 &        3.75 &      2.99 &        1.00 &        0.60 &        2.37 &      2.15 & 3.58\\
            \cmidrule(lr){1-2}\cmidrule(lr){3-6}\cmidrule(lr){7-10}\cmidrule(lr){11-14}\cmidrule(lr){15-15}\multicolumn{2}{r}{\textbf{Max Error \%}}&                 6.23 &        8.28 &        31.71 &      31.71 &           3.21 &        8.45 &        12.86 &      21.53 &        2.89 &        2.89 &        7.72 &      14.83 \\
            \bottomrule
            \end{tabular}
    \end{adjustbox}
    \caption{Number of clusters for k-means clustering results. Metrics are computed on scaled data and averaged over $10$ trials. Maximum standard error bounds for each metric by task and across tasks are included.}
    \label{tab:kmeanclust}
\end{table*}

\begin{table*}[t]
    \centering
    \begin{adjustbox}{max width=\textwidth}
        \begin{tabular}{clrrrrrrrrrrrrr}
        \toprule
                    \multicolumn{1}{c}{}&  & \multicolumn{4}{c}{\textbf{Logistic Regression}} & \multicolumn{4}{c}{\textbf{Random Forest}} & \multicolumn{4}{c}{\textbf{DNN}} \\
                    \cmidrule[0.75pt](lr){3-6}\cmidrule[0.75pt](lr){7-10}\cmidrule[0.75pt](lr){11-14}
                    \textbf{Dataset} & \textbf{$k$} &              \rot{Success} & \rot{Avg Success} & \rot{$\ell_2$ Dist.} & \rot{Diversity} &        \rot{Success} & \rot{Avg Success} & \rot{$\ell_2$ Dist.} & \rot{Diversity} &     \rot{Success} & \rot{Avg Success} & \rot{$\ell_2$ Dist.} & \rot{Diversity}& \rot{\textbf{Max Error \%}}\\
        \cmidrule[0.75pt](lr){1-2}\cmidrule[0.75pt](lr){3-6}\cmidrule[0.75pt](lr){7-10}\cmidrule[0.75pt](lr){11-14}\cmidrule[0.75pt](lr){15-15}
        \multirow{4}{*}{\textbf{\thead{Credit\\Card\\Default}}} & 1 &                 0.90 &        0.90 &        5.03 &      0.00 &           1.00 &        1.00 &        2.99 &      0.00 &        1.00 &        1.00 &        4.65 &      0.00 & 1.27\\
            & 2 &                 0.90 &        0.89 &        5.04 &      0.05 &           1.00 &        1.00 &        2.99 &      0.03 &        1.00 &        1.00 &        4.65 &      0.01 & 10.48\\
            & 3 &                 0.91 &        0.89 &        5.04 &      0.09 &           1.00 &        1.00 &        2.99 &      0.05 &        1.00 &        1.00 &        4.65 &      0.02 & 8.33\\
            & 4 &                 0.92 &        0.89 &        5.04 &      0.13 &           1.00 &        1.00 &        2.99 &      0.06 &        1.00 &        1.00 &        4.65 &      0.03 & 6.03\\
            & 5 &                 0.93 &        0.90 &        5.05 &      0.15 &           1.00 &        1.00 &        2.99 &      0.07 &        1.00 &        1.00 &        4.65 &      0.04 & 6.10 \\
            & 6 &                 0.94 &        0.90 &        5.06 &      0.18 &           1.00 &        1.00 &        2.99 &      0.09 &        1.00 &        1.00 &        4.65 &      0.04 & 6.53\\
            \cmidrule(lr){1-2}\cmidrule(lr){3-6}\cmidrule(lr){7-10}\cmidrule(lr){11-14}\cmidrule(lr){15-15}
        \multirow{4}{*}{\textbf{\thead{Give Me\\Some\\Credit}}} & 1 &                 0.58 &        0.58 &        2.37 &      0.00 &           1.00 &        1.00 &        2.83 &      0.00 &        1.00 &        1.00 &        4.11 &      0.00 & 8.67\\
            & 2 &                 0.59 &        0.58 &        2.37 &      0.03 &           1.00 &        1.00 &        2.82 &      0.04 &        1.00 &        1.00 &        4.11 &      0.00 & 32.44\\
            & 3 &                 0.60 &        0.58 &        2.39 &      0.05 &           1.00 &        1.00 &        2.81 &      0.06 &        1.00 &        1.00 &        4.11 &      0.01 & 31.38\\
            & 4 &                 0.62 &        0.59 &        2.39 &      0.07 &           1.00 &        1.00 &        2.82 &      0.07 &        1.00 &        1.00 &        4.11 &      0.01 & 30.58\\
            & 5 &                 0.63 &        0.59 &        2.40 &      0.08 &           1.00 &        1.00 &        2.83 &      0.09 &        1.00 &        1.00 &        4.11 &      0.01 & 24.86\\
            & 6 &                 0.66 &        0.59 &        2.40 &      0.09 &           1.00 &        1.00 &        2.80 &      0.10 &        1.00 &        1.00 &        4.11 &      0.01 & 25.59\\
        \cmidrule(lr){1-2}\cmidrule(lr){3-6}\cmidrule(lr){7-10}\cmidrule(lr){11-14}\cmidrule(lr){15-15}
        \multirow{4}{*}{\textbf{\thead{UCI\\Adult}}} & 1 &                 0.24 &        0.24 &        1.89 &      0.00 &           0.23 &        0.23 &        2.09 &      0.00 &        0.35 &        0.35 &        1.86 &      0.00 & 3.21\\
            & 2 &                 0.25 &        0.25 &        1.90 &      0.08 &           0.23 &        0.22 &        2.10 &      0.06 &        0.36 &        0.35 &        1.87 &      0.08 & 22.89\\
            & 3 &                 0.26 &        0.25 &        1.92 &      0.17 &           0.25 &        0.22 &        2.12 &      0.11 &        0.36 &        0.35 &        1.87 &      0.16 &9.92\\
            & 4 &                 0.27 &        0.25 &        1.93 &      0.20 &           0.25 &        0.22 &        2.13 &      0.16 &        0.37 &        0.35 &        1.89 &      0.20 & 11.87\\
            & 5 &                 0.28 &        0.25 &        1.94 &      0.28 &           0.26 &        0.22 &        2.13 &      0.19 &        0.37 &        0.35 &        1.90 &      0.27 & 9.85\\
            & 6 &                 0.28 &        0.25 &        1.95 &      0.31 &           0.26 &        0.22 &        2.15 &      0.23 &        0.37 &        0.35 &        1.90 &      0.26 & 7.62\\
            \cmidrule(lr){1-2}\cmidrule(lr){3-6}\cmidrule(lr){7-10}\cmidrule(lr){11-14}\cmidrule(lr){15-15}
                   \multicolumn{2}{r}{\textbf{Max Error \%}}&                 2,91 &        2.52 &        1.19 &      15.89 &           3.21 &        3.21 &        8.99 &      32.44 &        3.30 &        2.95 &        4.76 &      22.89 \\
            \bottomrule
            \end{tabular}
    \end{adjustbox}
    \caption{Number of clusters for random cluster assignment results. Metrics are computed on scaled data and averaged over $10$ trials. Maximum standard error bounds for each metric by task and across tasks are included.}
    \label{tab:randclust}
\end{table*}

\Step supports base clustering methods which allow the user to specify the number of clusters (e.g. $k$-means) and those which estimate a number of clusters (e.g. affinity propagation). 
With the latter, the clustering method determines $k$, the number of potential recourse paths \Step can produce. 

We suggest using clustering approaches where the user can specify the number of clusters (number of recourse paths), aligning with more traditional recourse methods. Methods that empirically determine a number of clusters to satisfy some objective function can result in a large $k$--while having multiple recourse paths is desirable, past a certain threshold loses its marginal utility, can overload users, and become challenging to interpret \emph{if one is considering all of those paths without additional post-processing}. 

The consideration around number of clusters has more to do with how the clusters are distributed in high-dimensional space relative to the learned decision boundary of each classifier, rather than as a function of “cluster goodness” or quality. From these results, we show that StEP is robust to the number of clusters for k-means, and therefore also robust to the relative size of each cluster.

We expect that clustering can negatively impact recourse when the clustering method and base model leverages immutable features. Intuitively, as an example within the Give Me Some Credit, one of the k clusters reflects an older/retired population unlikely to experience financial distress. Additionally, as observed in the
\CCHVAE analysis in \Cref{subsec:holistic}, only $8\%$ of positively classified training examples had a feature value {\em age} $\leq 59$ under logistic regression. Therefore,  \Step consistently fails to find a path to this cluster for PoIs younger than $59$ years old. This is reflected in \Step’s Avg Success under logistic regression and Give Me Some Credit in \Cref{tab:all-holistic}. However, the other clusters were reachable for most PoIs, meaning \Step almost always had a successful recourse path for every PoI.

From our empirical results, the consideration around number of clusters has more to do with how the clusters are distributed in high-dimensional space relative to the learned decision boundary of each classifier, rather than as a function of ``cluster goodness'' or quality. From these results, we show that \Step is robust to the number of clusters for k-means, and therefore also robust to the relative size of each cluster (\Cref{tab:kmeanclust}).

We also generated results with uniformly random cluster assignments, i.e., each point is randomly assigned a cluster value from $\{1, ..., k\}$ (\Cref{tab:randclust}). The $k=1$ case for both k-means and random clustering produce the same results (as expected). For random clustering, as $k$ increases, we observe small improvements across our metrics of interest and that \Step demonstrates reasonable performance on most tasks. We note that, given the strict categorical constraints in the UCI Adult task, some clusters are inherently less reachable for many PoIs.  

\subsection{Additional \Step User-Interference Results} \label{apdx:step-user-interference}

In Table \ref{tab:all-noise-appendix} present results on the path length, path steps, and proximal diversity metrics (defined in Section \ref{apdx:additional-comparative-analysis}). The noise $\beta$ we introduce offers insights into \Step's behavior. Interestingly, introducing noise produces improvements in performance on the Give Me Some Credit Dataset (\Cref{tab:all-noise}) by decreasin path length and path steps while providing a small boost in proximal diversity.
Recall that \Step clusters the {\em positively} labeled evaluation data which introduces additional hyperparameters: the number of clusters, $k$, and the clustering method. 
We cluster the evaluation data in full \emph{without removing any outliers} via sci-kit learn's off-the-shelf k-means algorithm. In particular, one cluster in the Give Me Some Credit dataset contained several outliers. 

When generating recourse for a PoI $\poi$, \Step produces a direction towards the centers of each of the $k$ clusters. 
Any outliers in the evaluation data can cause dramatic shifts in the directions, due to the sparse topology of the outlier clusters. 

The Give Me Some Credit dataset exhibits significant variance between the mean path length across the clusters. 
Beyond removing outliers from the dataset, we suggest the following to correct for the variance influenced by the presence of outliers: 
\begin{inparaenum}[(a)]
\item increase the number of clusters used and ignore the paths computed for a cluster which has an average path length of $\geq 2$ standard deviations higher than the mean path length between the other clusters; 
\item increasing the number of clusters may help disperse the cluster assignments of outliers, and the more clusters \Step uses results in more paths produced, which can help support generating a greater number of viable suggested recourse options for a given stakeholder.     
\end{inparaenum}

\section{Experimental Details for Reproducibility} \label{apdx:experimental-details}

\subsection{Datasets Details}\label{apdx:datasets}
\noindent {\bf Credit Card Default:} we produce random train, validation, and test sets from the $30,000$ instances using a 70/15/15 split, resulting in sets with approximately 21k/4.5k/4.5k datapoints respectively. 
This dataset contains $24$ features, $3$ of which are categorical features. We also convert the columns denoted by ``$\text{PAY}_*$'' to continuous values by replacing instances of $-1, -2$ with $0$. We set ``MARRIAGE'', ``AGE'', and ``SEX'' as immutable features. For \Step we set ``EDUCATION'' as an ordinal feature that can only increase.

\noindent {\bf Give Me Some Credit:} we produce randomly assigned train, validation, and test splits from the instances in the same manner as with Credit Card Default. Give Me Some Credit consists of $10$ features. We remove points in this dataset with missing values prior to splitting into train, validation and test sets. We set ``age'' as an immutable feature.

\noindent {\bf UCI Adult/Census Income:} we produce randomly assigned train, validation, and test splits from the instances in the same manner as with Credit Card Default. UCI Adult/Census Income consists of $15$ features, $6$ of which are categorical. We drop the ``education'' feature as it is equivalent to the ``education-num'' feature. For \Step we set ``education-num'' as a feature that can only increase when giving recourse. We set the following features as immutable features: ``age'', ``marital-status'', ``relationship'', ``race'', ``sex'', ``native-country''.

\subsection{Base Model Details}\label{apdx:models}
Each method relies on a base machine learning model; in our experiments, we use sci-kit learn's implementation of logistic regression and random forest replicating practitioners' preferences for simpler and more explainable models. To evaluate more complex models, we use the PyTorch library implement a non-linear neural network model with two hidden layers of 16 and 32 neurons. We use simple holdout set validation to determine the hyperparameters used, outlined in \ref{apdx:hyperparams}.

\subsection{\Step Parameters Discussion}

\paragraph{Choosing a Distance Metric:}
We use the simple rotation invariant $\ell_2$-norm as our distance metric. 
To ensure the distance function is not biased towards any feature, we normalize continuous feature values with respect to their mean and standard deviation (\citet{mothilal2020explaining} follow a similar methodology). More formally, for a continuous feature $i$, we transform its value as follows:
\begin{align*}
    x^j_i := \frac1{\sigma_i}(x^j_i - \mu_i)
\end{align*}
where $\mu_i$ is the mean value of the feature and $\sigma_i$ is the standard deviation.

\paragraph{Choosing the $\alpha$ function}
We propose two different $\alpha$  functions --- the {\em volcano} function\footnote{so named for its shape when drawn on a whiteboard.} and the {\em sloped} function.
The volcano function weighs nearby points higher than faraway points, but all points closer than a specific threshold are weighed equally (a similar function is used by \citet{patel2020explanation}). We denote this function by $\alpha_v$ and define it as follows:
\begin{align}
         \alpha_{v}(z ; d, \gamma) = 
               \begin{cases} 
                  \frac{1}{z^d} & z > \gamma \\
                  \frac{1}{\gamma^d} & z \leq \gamma
               \end{cases}\label{eq:alpha-volcano}
    \end{align}
The sloped function is shaped like a normal distribution curve with a $0$ mean. We denote this function by $\alpha_s$ and define it as follows
\begin{align*}
    \alpha_s(z ; w) = \exp\left( -\frac{1}{2} \left(\frac{z}{w}\right)^2 \right)
\end{align*}

Both of these functions are continuous and upper bounded. If we divide the output of these functions by $z$ (the input to the functions), the output will satisfy the conditions of Theorem \ref{thm:step-privacy}, i.e. the \Step algorithm can be made differentially private with a slightly modified version of these $\alpha$ functions.

In our experiments, we use the volcano function $\alpha_v$ with $d = 2$ and $\gamma = 0.5$.

\subsection{Other Hyperparameter Selection} \label{apdx:hyperparams}

We perform a basic grid search across base models, recourse methods, and datasets using sci-kit learn's GridSearchCV function. To determine appropriate hyperparameters to use across the base ML models and recourse methods, we roughly optimize for a reasonably high success rate and to minimize distance. 

For each base ML model, we sweep over a confidence cutoff between $\{0.50,0.55,0.60,0.65, 0.7\}$ and evaluate the performance of each recourse method. We find that $0.7$ provides reasonable results across datasets, base models, and recourse methods while remaining a reasonable value a practitioner may use. We vary $k \in \{1, 2, 3, 4, 5\}$, the number of paths to produce for each PoI (and for \Step, the number of clusters to produce), and fix $k=3$ for our comparative analysis and user-interference experiments. 

For \Step, we consider step sizes in $\{0.10,0.25,0.50,0.75, 1.00\}$ and fix a value of $1$ across all experiments. \FACE has two main hyperparameters: graph type in $\{k\text{-NN}, \epsilon, \text{KDE} \}$ and distance threshold (a float). In \ref{apdx:face-discuss}, we provide substantial discussion on the distance threshold parameter. For all our experiments using \FACE, we used the $\epsilon$ graph following a suggestion from the authors of the method. For \CCHVAE we set the step distance hyperparameter to $1$. We detail this hyperparameter in \ref{apdx:baselines}. For all applicable recourse methods, we allow a maximum of $50$ iterations to produce $k$ counterfactual(s).

\subsection{Computational Resources Used}
We used two machines for our empirical evaluation, including for base model training (i.e.\ logistic regression, random forest, neural network), for all recourse experiments, and post-processing of results to produce metrics. These machines contain:
\begin{itemize}
\item
$8$ CPU cores, $64$GB RAM, NVIDIA 4070 GPU, and $2$TB local SSD disk memory.
\item $16$ CPU cores $64$GB RAM, NVIDIA 4080S GPU, and $4$TB local SSD disk memory.
\end{itemize}
\end{document}